\documentclass[12pt]{article}
\usepackage{graphicx,multirow,color}
\usepackage{algorithm,algorithmic}
\usepackage{amsfonts,amssymb}
\usepackage[namelimits]{amsmath}
\usepackage{mathrsfs} 
\usepackage[numbers,sort&compress]{natbib}
\newtheorem{theorem}{Theorem}[section]
\newtheorem{proposition}[theorem]{Proposition}
\newtheorem{problem}{Problem}[section]
\newtheorem{proof}{Proof}[section]
\newtheorem{remark}{Remark}[section]

\newtheorem{acknowledgements}{Acknowledgements}[section]

\newcommand{\prox}{\ensuremath{\text{\rm prox}}}
\newcommand{\abs}{\ensuremath{\text{\rm abs}}}
\newcommand{\sign}{\ensuremath{\text{\rm sign}}}
\newcommand{\RR}{\mathbb{R}}

\newsavebox{\results}

\begin{document}

\title{Sparse hierarchical interaction learning with epigraphical projection\thanks{This is a pre-print of an article published in Journal of Signal Processing Systems. The final authenticated version is available online at: https://doi.org/10.1007/s11265-019-01478-1. Corresponding author N. Pustelnik \textrm{nelly.pustelnik@ens-lyon.fr}.}
}


\author{Mingyuan Jiu  $^\ddag$\thanks{School of Information Engineering, Zhengzhou University, Zhengzhou, 450001, China. This work was supported by a grant from the National Natural Science Foundation of China (No.~61806180)} \and
        Nelly Pustelnik \thanks{Univ Lyon, Ens de Lyon, Univ Lyon 1, CNRS, Laboratoire de Physique, Lyon, 69342, France}\and
	Stefan Janaqi  \thanks{TOTAL SA, Direction Scientifique, 92069 Paris La D\'efense, France}\and
	M\'eriam Chebre  $^\S$\and
        Lin Qi  $^\dag$\and 
	Philippe Ricoux  $^\S$
}



\date{26 August 2019}

\maketitle

\begin{abstract}
This work focuses on learning optimization problems with quadratical interactions between variables, which go beyond the additive models of traditional linear learning. We investigate more specifically two different methods encountered in the literature to deal with this problem: ``hierNet" and structured-sparsity regularization, and study their connections. We propose a primal-dual proximal algorithm based on an epigraphical projection to optimize a general formulation of these learning problems. The experimental setting first highlights the improvement of the proposed procedure compared to state-of-the-art methods based on fast iterative shrinkage-thresholding algorithm (i.e.~FISTA) or alternating direction method of multipliers (i.e.~ADMM), and then, using the proposed flexible optimization framework,  we provide fair comparisons between the different hierarchical penalizations and their improvement over the standard $\ell_1$-norm penalization. 
The experiments are conducted both on synthetic and real data, and they clearly show that the proposed primal-dual proximal algorithm based on epigraphical projection is efficient and effective to solve and investigate the problem of  hierarchical interaction learning.
\end{abstract}

\section{Introduction} \label{sec:intro}

Learning interactions between features is of great interest in data processing.  In this work we focus on quadratic interaction effects between the features that extend linear models. Our work focuses both on the \textit{regression problem} and \textit{multiclass SVM}, which are still subjects of active research when small dataset are involved (especially in medical~\cite{SPILKA:2017:A} or physics~\cite{PASCAL:2018:A} applications).\\

\noindent \textbf{Feature selection} -- Our work follows the line of feature selection methods~\cite{Bach_F_2012_j-ftml_optimization_sip,Blondel2013, Chakraborty2015, Laporte2014, Chierchia2015_a,Xu2017}, which aim to remove the redundant features and only keep the informative ones. The relevant features are usually correlated and belong to the same group. 
In this work, additionally to perform feature selection, we learn the interactions between the selected features by making some hierarchy constraints, where the interactions may occur either between the selected features or if only one of both features is active.
\\

\noindent \textbf{Knowledge from the interactions} -- Features interactions can provide us a new knowledge about the correlation between the subjects. For instance, we can refer to gene interactions aiming to highlight relevant regulatory relationships for genes~\cite{Pirayre_A_2015_j-bmb-bioinformatics_bra_cbr} or as mentioned in \cite{Bien2013} ``the co-occurrence of two symptoms may lead a doctor to be confident that a patient has a certain disease, whereas the presence of either symptom without the other would provide only a moderate indication of that disease''. \\

\noindent \textbf{Better discrimination} --  Dealing with interactions allows us to increase the feature space  in order to provide a better discrimination in the learning process~\cite{Anderson1975}. The integration of quadratic interactions in the learning problem is not new, for example, discriminant quadratic learning \cite{Anderson1975, Randles1978, Witten2009} learns the covariance matrix in the quadratic term to improve the discrimination ability.  {We can also refer to efficient multilevel procedures~\cite{Agarwal_A_2014,HaoFengZhang2018} starting with fast learning algorithm focus on the linear model and then adding higher-order interaction features, possibly using the learned weights as a guide. }\\

\noindent \textbf{Dimensionality challenge and sparsity} -- One major challenge  when dealing with quadratic interaction learning is that the interaction number quadratically increases with the feature size, for example,  {1000 features will have about one million possible interactions}, this therefore results in an overfitting problem due to insufficient observation samples in most of the regression problems and for some classification ones. To overcome this weakness, the linear weights and the quadratic interactions can be assumed to be sparse because most features would not contribute to the decision. Sparsity-based regularization is known to be mainly suitable when the feature size is larger than the training samples. Various sparsity regularizations have been proposed and extensively studied in the additive model context, for instance, involving $\ell_0$-pseudo-norm \cite{Weston2003}, $\ell_1$-norm \cite{Tibshirani1996}, $\ell_\infty$-norm \cite{Zou2008},  or structural sparsity norm \cite{Bach2012a}. See \cite{Bach_F_2012_j-ftml_optimization_sip} for an exhaustive list of sparse-based regularization in learning and \cite{Gui2016} to refer to structural sparsity. The formalism we consider in this work follows ideas proposed in  \cite{Zhao_P_2009_j-ann-stat_com_apf, Bien2013,Haris2014,Lim2015,She_Y_2016_j-am-stat-ass_gro_res}.  A detailed comparison with these state-of-the-art methods is provided in next section. \\

\noindent \textbf{Contributions} --  Quadratic interaction learning imposes a specific design of the penalization term (this will be formally described in Section~\ref{sec:model}).  Several choices of penalizations have already been proposed in the literature, but each of them relies on a different algorithmic strategy, leading to unfair numerical comparisons. Indeed, do the numerical differences in the learning performance come from the penalization choice or from the algorithmic schemes ? For instance, it is well known that the number of inner iterations (as proposed in \cite{Jenatton2011b}) is often a parameter tricky to adjust, or that the algorithmic parameters such as the step-size may impact a lot the convergence speed \cite{Haris2014}. For all these reasons, this work is dedicated to provide  a common algorithmic scheme without inner iterations aiming to handle with the most recents state-of-the-art penalizations. Our detailed contributions are listed below:
\begin{itemize}
 \item  {Propose a new algorithmic scheme relying of primal-dual algorithm and derive new closed form of epigraphical projections in order to solve the unifying minimization problem (cf. Problem~\ref{pb:main}) proposed in \cite{She_Y_2016_j-am-stat-ass_gro_res}  when $q=+\infty$ and $z(\cdot) = \Vert \cdot \Vert_r$ with $r=\{1,+\infty\}$;} 
 \item  {Compare the proposed algorithms to the state-of-the-art methods (FISTA and ADMM), showing a unique and efficient algorithmic scheme for weak and strong formulation and  for $r=\{1,+\infty\}$ allowing us to avoid inner iterations.}
 \item Provide the counterpart of the regression minimization Problem~\ref{pb:main} for classification task. 
 \item  {Conduct extensive experiments in the regression and classification framework on the simulated data, on two disease applications (i.e.~HIV and Parkinson) to analyze the correlations between the factors for the diseases, validating the effectiveness and efficiency of the proposed algorithms.}
\end{itemize}
This work improves over our preliminary contribution \cite{Jiumlsp2018}  where we restricted our study to classification (cf. section~\ref{sec:multiclassSVM} of this work). The algorithm was  derived for $r=1$, relying only on Proposition~\ref{prop:epil1} and for which the proof was not provided. The experiments were conducted on face classification while in this work we focus on synthetic data and diseases applications. \\

\noindent \textbf{Outline} --This paper is organized as follows: Section~\ref{sec:model} firstly describes the formal problem in the context of regression and refers to related works. Section~\ref{sec:relatedwork} derives an epigraphical writing of the objective function. Section~\ref{sec:frdual} presents the primal-dual proximal algorithm iterations and the convergence guarantees. Its specification to our minimization problem with the hierarchical regularization term is specified and new epigraphical projections are provided. Section~\ref{sec:multiclassSVM} provides the objective function and the associated algorithm in the context of multiclass SVM with hierarchical interactions.
Section~\ref{sec:experiments} evaluates the performance of the proposed strategy compared to FISTA and ADMM formulations proposed in \cite{Jenatton2011b,Bien2013} both on synthetic data and real applications. Finally Section~\ref{sec:conclusion} gives our conclusion.

\section{Regression model} \label{sec:model}
The regression training set is denoted $\mathcal{R} = \{ \big( {y}_\ell, \phi({\bf x}_\ell) \big) \in \mathbb{R} \times \mathbb{R}^{N}\,\vert\, \ell\in \{1,\ldots,L\}\}$,  {where ${\bf x}_\ell$ can denote either data without structure or with structure such as a signal, an image or a graph of size $M$. Then $\phi({\bf x}_\ell)$ denotes the coefficients associated with the data ${\bf x}_\ell$, e.g time-frequency coefficients~\cite{Flandrin_P_1999}, scattering coefficients~\cite{Bruna_J_2013} or CNN coefficients~\cite{LeCun_Y_1995}, where generally $N\gg M$.} 

Several minimization strategies have been provided in the literature to jointly select discriminating features among the features $\phi({\bf x}_\ell)$ and identify meaningful interactions between these features. To clarify the state-of-the-art contributions in this context, we propose to recall a general minimization problem derived in \cite{She_Y_2016_j-am-stat-ass_gro_res}:
\begin{problem} Let $\mathcal{R} = \{ \big( {y}_\ell, \phi({\bf x}_\ell) \big) \in \mathbb{R} \times \mathbb{R}^{N}\,\vert\, \ell\in \{1,\ldots,L\}\}$ be the training regression dataset. We aim to estimate the regression weight vector $\widehat{\mathbf{v}}$ and the interaction matrix $\widehat{\mathbf{\Theta}}$ solving:
\begin{equation}
(\widehat{\mathbf{v}}, \widehat{\mathbf{\Theta}}) \in \underset{\substack{\mathbf{v}\in \mathbb{R}^N\\\Theta\in \mathbb{R}^{N\times N}}}{\textrm{arg min}}\; \frac{1}{2} \sum_{\ell=1}^L  \big( y_\ell - \phi({\bf x}_\ell)^\top \mathbf{v} - \phi({\bf x}_\ell)^\top \mathbf{\Theta} \phi({\bf x}_\ell) \big)^2 + \Omega(\mathbf{v},  \mathbf{\Theta} ) \label{equa:hier0}
\end{equation}
where, for every $\mathbf{v} = \big(v^{(i)}\big)_{1\leq i \leq N}$  and $ \mathbf{\Theta}  = \big(\Theta^{(i,j)}\big)_{1\leq i \leq N, 1\leq j \leq N}$,
\begin{equation} 
\Omega(\mathbf{v}, \Theta) = \frac{\lambda}{2} \Vert \Theta \Vert_1 +  \lambda \sum_{i=1}^N \Vert \big(v^{(i)}, z\big(\Theta^{(i,\cdot)}\big)\big)\Vert_q  + \iota_C(\Theta) 
 \label{equa:gencons}
\end{equation}
with $z \colon \mathbb{R}^N \to \mathbb{R}$, $1\leq q \leq +\infty$, and $\iota_C$ denotes the indicator function\footnote{For every $x\in \mathcal{H} $, $\iota_C(x) = 0$ if $x\in C$ and $+\infty$ otherwise.} of a closed convex set $C\subset \mathbb{R}^{N\times N}$.
\label{pb:main}
\end{problem}

\subsection{Choice for the constraint $C$: Weak/strong hierarchy}
The hierarchy penalization $\Omega$ creates the relationship between $\mathbf{v}$ and $ \mathbf{\Theta} $, as shown in Fig.~\ref{fig:interactiondemo}. 
Two types of  hierarchy constraints have been defined in ``hierNet" \cite{Bien2013}: (i) \textit{weak hierarchy} where the interactions happen (i.e., $\Theta^{(i,j)}\neq 0$)  {only} if one of the associated  features in $\mathbf{v}$ is non-zero (i.e., $v^{(i)}\neq 0$ or $v^{(j)}\neq 0$) and (ii) \textit{strong hierarchy} when both associated features are non-zero.   These two types of hierarchy are imposed by means of the closed convex set $C\subset \mathbb{R}^{N\times N}$. The weak hierarchy is obtained with  $C= \mathbb{R}^{N\times N}$ and strong hierarchy by imposing a symmetric structure for the matrix of interactions using $C= S =  \{ \mathbf{\Theta}  \in \mathbb{R}^{N\times N} \,\vert \,  \mathbf{\Theta}  =  \mathbf{\Theta} ^\top\}$.

\begin{figure}[tbp]
\centering
\includegraphics[scale=0.4]{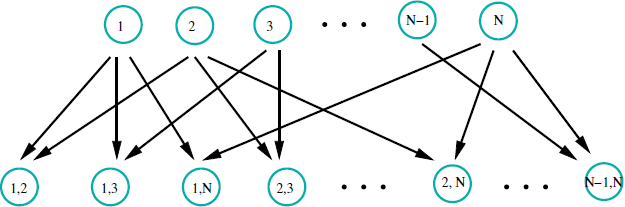}
\caption{Hierarchical feature interaction. The first row is $N$ additive features, and the second row is the interactions between pairs of features.} \label{fig:interactiondemo}
\end{figure}
\subsection{Positioning of Problem~\ref{pb:main} w.r.t state-of-the-art} 
Most of the state-of-the-art penalizations may be interpreted as a specific case of Problem~\ref{pb:main}, for instance:
\begin{itemize}
\item \cite{Jenatton2011b} :  $z(\cdot) =  (\cdot)^\top$,  $q=\{2,+\infty\}$,
\item \cite{Bien2013,Haris2014} : $z(\cdot)  = \Vert \cdot \Vert_1$ and $q=+\infty$,
\item \cite{She_Y_2016_j-am-stat-ass_gro_res} : $z(\cdot) =  (\cdot)^\top$ and $q=\{2\}$.
\end{itemize}
Note that \cite{Jenatton2011b} focuses on weak hierarchy using $C = \mathbb{R}^{N\times N}$. \cite{Bien2013} is encountered under the name ``hierNet", while \cite{Haris2014}  is named ``FAMILY" and \cite{She_Y_2016_j-am-stat-ass_gro_res} is ``Type-A GRESH". The latent overlapping group lasso formulation of \cite{Bien2013} is known as ``glinternet" \cite{Lim2015}.

From the algorithmic point of view several strategies can be encountered. In \cite{Jenatton2011b} and \cite{She_Y_2016_j-am-stat-ass_gro_res}, the iterations are derived from FISTA \cite{Beck_A_2009_j-siam-is_fast_istalip}. This procedure appears to be very efficient when $C = \mathbb{R}^{N\times N}$ while for $C=S$ it requires inner iterations based on Dykstra's algorithm that significantly slows the convergence \cite{Chaux_C_2009_siims_Nested_iafccirp}. On the other hand, \cite{Bien2013} and \cite{Haris2014} resort to ADMM to deal with $C=S$. 

Using recent developments of non-smooth convex optimization, the objective of this work is to provide a unified framework and an efficient algorithmic strategy based on primal-dual proximal algorithms and epigraphical projection to solve an epigraphical relaxation of~\eqref{equa:gencons} when $C = \mathbb{R}^{N\times N}$ or $C=S$, $z(\cdot)  = \Vert \cdot \Vert_r$ with $r=\{1,+\infty\}$  and $q=+\infty$.
The algorithmic procedure does not resort to inner iterations, thus leading to a faster implementation.\footnote{Matlab codes will be made available at the time of the publication.}

\section{Epigraphical formulation} \label{sec:relatedwork}

 {In order to handle with complex optimization problems, the two major solutions encountered in the literature are either to formulate it in the dual or to increase the dimensionality of the problem. Dealing with an epigraphical formulation belongs to this second class and it is mainly adapted to nonlinear constraints. We recall that the epigraph of a function $f\colon \mathcal{H}\to ]-\infty,+\infty]$ is defined as: $\mathrm{epi} f = \{ (x, \zeta)\in \mathcal{H} \times \mathbb{R}\,\vert\, f(x)\leq \zeta \}$, the projection of a point to this epigraph is called epigraphical projection. The utility of such epigraphical formulation can be illustrated when one wants to deal with a constraint of the form
$\Vert x\Vert_1\leq \eta$, for every $x = (x_i)_{i\in \mathbb{I}\subset \mathbb{N}}$ and $\eta>0$ is known, whose projection does not have a closed form expression. An alternative is then to replace this constraint with these two constraints: $\vert x_i\vert \leq \zeta_i$ and $\sum_{i\in \mathbb{I}} \zeta_i = \eta$ where the first one denotes an epigraphical constraint and the second one denotes an hyperplane constraint, both having a closed form expression for the associated projection~\cite{Chierchia2015}. }

The resolution of Problem~\ref{pb:main} is difficult, since it involves non-smooth convex functions and symmetry constraints. In \cite{Jenatton2011b}, Jenatton et al. solve the problem for $r=\infty$ in the weak hierarchy formulation by means of a proximal algorithm. However, their algorithm is not well adapted to strong hierarchy. In \cite{Bien2013}, the authors reformulate Problem~\ref{pb:main} when $r=1$ under an epigraphical formulation  which represents the problem as a set of hard constraints. This reformulation helps in the interpretation and it highlights a reduction in the shrinkage of certain main effects and an increase in the shrinkage of certain interactions. Its second advantage is to help in the design of the ADMM algorithmic scheme. 
The following proposition gives the epigraphical formulation of the minimization problem considered in this work.

\begin{proposition} \label{prop:epi}The minimization problem
\begin{equation}
\label{eq:regform}
 \underset{\mathbf{v}, \mathbf{\Theta} }{\textrm{minimize}}  \; \frac{1}{2} \sum_{\ell=1}^L  \big( y_\ell - \phi^\top({\bf x}_\ell) \mathbf{v} - \phi^\top({\bf x}_\ell) \mathbf{\Theta} \phi({\bf x}_\ell) \big)^2  + \Omega(\mathbf{v},  \mathbf{\Theta} )
\end{equation}
where, for every $\mathbf{v} = \big(v^{(i)}\big)_{1\leq i \leq N}$  and $ \mathbf{\Theta}  = \big(\Theta^{(i,j)}\big)_{1\leq i \leq N, 1\leq j \leq N}$,
\begin{equation} 
\Omega(\mathbf{v}, \Theta) = \frac{\lambda}{2} \Vert \mathbf{\Theta} \Vert_1 +  \lambda \sum_{i=1}^N \max \{v^{(i)},  \Vert \Theta^{(i,\cdot)}\Vert_r\}  + \iota_C(\mathbf{\Theta}) 
 \label{equa:gencons}
\end{equation}
can be written in an epigraphical framework as
\begin{multline}
\label{eq:equivform}
 \underset{(\mathbf{v}^+, \mathbf{v}^-,\mathbf{\Theta})\in O\times O \times C }{\textrm{minimize}}  \; \frac{1}{2} \sum_{\ell=1}^L \big( y_\ell  -   \phi^\top({\bf x}_\ell) (\mathbf{v}^+ - \mathbf{v}^- ) - \phi^\top({\bf x}_\ell) \mathbf{\Theta} \phi({\bf x}_\ell) \big)^2 \\+  \frac{\lambda}{2} \Vert  \mathbf{\Theta}  \Vert_1 +  \lambda 1^\top(\mathbf{v}^+ +\mathbf{v}^-) +\sum_{i=1}^N \iota_{E_r}\big(v^{+(i)},v^{-(i)}, \Theta^{(i,\cdot)}\big)
\end{multline}
where $O$ denotes the positive orthant, $\mathbf{v} = \mathbf{v}^+ - \mathbf{v}^-$ and  ${E_r}= \{(\omega^+,\omega^-,  {\bf u}) \in \mathbb{R}\times\mathbb{R}\times\mathbb{R}^N\,\vert \, \Vert {\bf u} \Vert_r \leq \omega^++\omega^-\}$ with $r=\{1,+\infty\}$.
\end{proposition}
\begin{proof}
  {The proof is given in Appendix A.} 
\end{proof}

The algorithmic strategy designed in next section will focus on this epigraphical formulation \eqref{eq:equivform}. 

\section{Primal-dual proximal algorithm} \label{sec:frdual}
Proximal algorithms are derived from two main frameworks that are the forward-backward scheme and the Douglas-Rachford iterations (both being deduced from the Krasnoselskii-Mann scheme) \cite{BauschkeCombettes2011}. FISTA can be presented as an accelerated version of forward-backward iterations \cite{Beck_A_2009_j-siam-is_fast_istalip,Chambolle_A_2015_jota_con_ifa} while ADMM can be viewed as a  {Douglas-Rachford procedure} in the dual \cite{Setzer_S_2009_ssvm_spl_bad}. 
In the literature dedicated to sparse regression the most encountered algorithm is FISTA when dealing with a sum of two convex functions where one is differentiable with a Lipschitz gradient. When the criterion involves more than two functions, typically an additional constraint such as the constraint $S$ defined previously, is added and most of the works derive an ADMM procedure or compute the proximity operator by means of inner iterations which is often known to leading to an approximate solution even if global convergence can be obtained in specific cases \cite{Chaux_C_2009_siims_Nested_iafccirp}. Another class of algorithmic procedures allowing to minimize a 
criterion with more than two functions, possibly including a differentiable function with a Lipschitz gradient, is the class of primal-dual proximal approaches~\cite{Chambolle2011a,Komodakis_N_2015,Condat_L_2012,Vu2013}. 

Several primal-dual proximal schemes have been derived but one of the most popular is based on forward-backward iterations \cite{Condat_L_2012,Vu2013} in order to estimate
\begin{equation}
\label{eq:primal}
\widehat{\mathbf{w}}\in \underset{\mathbf{w} \in \mathcal{H}}{\textrm{Argmin}} \; f(\mathbf{w})  + g(\mathbf{w})  + h(H\mathbf{w}),
\end{equation}
\noindent where $H\colon \mathcal{H} \to\mathcal{G}$ denotes a bounded linear operator, $\mathcal{H}$ and $\mathcal{G}$ being two Hilbert spaces, $f\colon\mathcal{H}\to ]-\infty,+\infty]$, $g\colon\mathcal{H}\to ]-\infty,+\infty]$ and $h\colon\mathcal{G}\to ]-\infty,+\infty]$ denote convex, l.s.c and proper functions.  {We additionally assume that $f$ is differentiable and $\nabla f$ has a Lipschitz constant denoted as $\beta>0$}. 
Iterations are summarized in Algorithm~\ref{algo:primaldual} involving the proximity operator defined as $$(\forall \mathbf{w}\in \mathcal{H}) \quad \prox_g(\mathbf{w}) = \arg \min_{\mathbf{u}\in \mathcal{H}}  {\frac{1}{2}}\Vert \mathbf{u} - \mathbf{w} \Vert_2^2 + g(\mathbf{u})$$ and $h^{*}$ denotes the Fenchel-Rockafellar conjugate  of function $h$. The proximity operator of the conjugate can be computed according to the Moreau identity that is $\textrm{prox}_{\sigma h^{*}}(\mathbf{w}) = \mathbf{w} - \sigma \textrm{prox}_{h/\sigma}(\mathbf{w}/\sigma)$ for $\sigma>0$.  {The sequence $(\mathbf{w}^{[k+1]})_{k\in \mathbb{N}}$ converges to a minimizer $\widehat{\mathbf{w}}$ of \eqref{eq:primal}. Moreover, the sequence $(\mathbf{u}^{[k+1]})_{k\in \mathbb{N}}$ converge to a minimizer of the dual formulation of \eqref{eq:primal} that is 
$$
\widehat{\mathbf{u}}\in \underset{\mathbf{u} \in \mathcal{G}}{\textrm{Argmin}} \; (f+ g)^*(\mathbf{-H^*u})  + h^*(\mathbf{u}).
$$}

This algorithmic scheme can be extended for dealing with more than three functions by setting  $h(Hx) = \sum_{k=1}^K h_k(H_k x)$ involving the computation of $\textrm{prox}_{\sigma h_k^{*}} $. The interest of this scheme compared to ADMM is twofold: it first makes the possibility to deal with the gradient of the differentiable function and secondly it avoids  to invert $\sum_k H_k^* H_k$.
 
\begin{algorithm}[H] 
\small
\caption{Primal-dual splitting algorithm} \label{algo:primaldual}
Parameter settings: Set $\tau>0$, $\sigma>0$,  {such that $\frac{1}{\tau} - \sigma\Vert H \Vert^2 \geq \frac{\beta}{2}$}\\
Initialization:  $(\mathbf{w}^{[0]}, \mathbf{u}^{[0]})\in \mathcal{H}\times \mathcal{G}$ \\
For $k=0, 1, \ldots$ 
\begin{equation}
 \left\lfloor \begin{array}{l}  
\mathbf{w}^{[k+1]} = \textrm{prox}_{\tau g} \big(\mathbf{w}^{[k]} - \tau \nabla f(\mathbf{w}^{[k]}) - \tau H^{*} \mathbf{u}^{[k]} \big) \\
\mathbf{u}^{[k+1]} = \textrm{prox}_{\sigma h^{*}} \big( \mathbf{u}^{[k]} + \sigma H (2\mathbf{w}^{[k+1]} - \mathbf{w}^{[k]}) \big)
\end{array}\right. \nonumber
\end{equation}
\end{algorithm}

\subsection{Specificity for minimizing \eqref{eq:equivform}} 

We first provide the iterations of Algorithm~\ref{algo:primaldual} specified to the minimization of~\eqref{eq:equivform}. By setting $\mathbf{w} = (\mathbf{v}^+,\mathbf{v}^-, {\bf{\Theta}}) \in \mathbb{R}^N \times\mathbb{R}^N \times \mathbb{R}^{N\times N}$, for weak hierarchy, we can split it as follows:
\begin{equation} \label{equa:weakgeneraldecompo}
\begin{cases}
f(\mathbf{w}) = \frac{1}{2} \sum_{\ell=1}^L  \big( y_\ell - \phi^\top({\bf x}_\ell) (\mathbf{v}^+ - \mathbf{v}^-) - \phi^\top({\bf x}_\ell) \mathbf{\Theta} \phi({\bf x}_\ell) \big)^2 + \lambda 1^\top(\mathbf{v}^+ + \mathbf{v}^-),\\
g(\mathbf{w}) = \iota_{O}(\mathbf{v}^+)  +\iota_{O}(\mathbf{v}^-) +  \frac{\lambda}{2} \Vert \mathbf{\Theta} \Vert_1,\\
h(\mathbf{w}) = \sum_{i=1}^N \iota_{E_r}\big(v^{+(i)},v^{-(i)}, \Theta^{(i,\cdot)}\big),
\end{cases}
\end{equation}
and for strong hierarchy:
\begin{equation} \label{equa:generaldecompo}
\begin{cases}
f(\mathbf{w}) = \frac{1}{2} \sum_{\ell=1}^L  \big( y_\ell - \phi^\top({\bf x}_\ell) (\mathbf{v}^+ - \mathbf{v}^-) - \phi^\top({\bf x}_\ell) \mathbf{\Theta} \phi({\bf x}_\ell)\big)^2 + \lambda 1^\top(\mathbf{v}^+ + \mathbf{v}^-),\\
g(\mathbf{w}) = \iota_{O}(\mathbf{v}^+)  +\iota_{O}(\mathbf{v}^-) + \iota_{S}(\mathbf{\Theta}), \\
h_1(\mathbf{w}) = \frac{\lambda}{2} \Vert \mathbf{\Theta} \Vert_1, \\
h_2(\mathbf{w}) = \sum_{i=1}^N \iota_{E_r}\big(v^{+(i)},v^{-(i)}, \Theta^{(i,\cdot)}\big).
\end{cases}
\end{equation}
The respective iterations are summarized in Algorithm~\ref{algo:weakprimaldualspecif} and \ref{algo:primaldualspecif}. The projection onto the positive orthant and on $S$ have well known closed form expressions \cite{BauschkeCombettes2011} that are:
$$
\begin{cases}
P_O(\cdot) = \max\{0, \cdot\},\\
P_S(\mathbf{\Theta}) = \frac{\mathbf{\Theta} + \mathbf{\Theta}^\top}{2}.
\end{cases}
$$
The difficulty comes from the computation of $P_{E_r}$ whose expressions are provided in the next section.

\begin{algorithm*}
\small
\caption{-- Weak-PD-$\ell_r$ -- Primal-dual splitting algorithm to solve \eqref{eq:equivform}  when $C = \mathbb{R}^{N\times N}$} \label{algo:weakprimaldualspecif}
Parameter settings: Set $\beta = 2\sum_{\ell=1}^L \vert \phi(\mathbf{x}_{\ell}) \vert^2 + \sum_{\ell=1}^L \vert \phi(\mathbf{x}_{\ell})^{\top} \phi(\mathbf{x}_{\ell})\vert^2$, $\tau>0$, $\sigma>0$,  {such that $\frac{1}{\tau} - \sigma\geq \frac{\beta}{2}$ }. \\
Initialization: $ (\mathbf{v}^{+[0]}, \mathbf{v}^{-[0]}, \mathbf{\Theta}^{[0]}, \mathbf{s}^{+[0]}, \mathbf{s}^{-[0]}, \mathbf{\Lambda}^{[0]})\in \mathbb{R}^N \times\mathbb{R}^N \times\mathbb{R}^{N \times N} \times \mathbb{R}^N \times \mathbb{R}^N \times \mathbb{R}^{N \times N} $. \\
For $k=0, 1, \ldots$ 
\begin{equation}
 \left\lfloor \begin{array}{l}  
b_\ell = y_\ell  -  \phi^\top({\bf x}_\ell) (\mathbf{v}^{+[k]} - \mathbf{v}^{-[k]})- \phi^\top({\bf x}_\ell) \mathbf{\Theta}^{[k]} \phi({\bf x}_\ell)\\
\mathbf{v}^{+[k+1]} = P_{O} \big(\mathbf{v}^{+[k]} + \tau \sum_\ell \phi({\bf x}_\ell)  b_\ell - \tau \lambda - \tau \mathbf{s}^{+[k]}  \big) \\
\mathbf{v}^{-[k+1]} = P_{O} \big(\mathbf{v}^{-[k]} - \tau \sum_\ell \phi({\bf x}_\ell)  b_\ell - \tau \lambda - \tau \mathbf{s}^{-[k]}\big) \\
\mathbf{\Theta}^{[k+1]} = \textrm{ {prox}}_{\frac{\tau \lambda}{2} \Vert \cdot \Vert_1} \big(\mathbf{\Theta}^{[k]} + \tau \big(\sum_\ell \phi^{(i)}({\bf x}_\ell)  \phi^{(j)}({\bf x}_\ell)  b_\ell\big)_{1\leq i\leq N,1\leq j\leq N}  - \tau \mathbf{\Lambda}^{[k]} \big) \\
\textrm{For $i=1,\ldots,N$ }\\
\lfloor\;\;\; (s^{+(i)[k+1]},s^{-(i)[k+1]},\Lambda^{(i,\cdot)[k+1]})  \\
 \qquad\;\; = \textrm{prox}_{\sigma \iota_{E_r}^{*}} \big(  s^{+(i)[k]}  + \sigma  (2{v}^{+[k+1]} - {v}^{+[k]}),  s^{-(i)[k]}  + \sigma  (2{v}^{-(i)[k+1]} - {v}^{-(i)[k]}), \\
 \qquad\;\; \Lambda^{+(i,\cdot)[k]} + \sigma  (2\Theta^{(i,\cdot)[k+1]} - \Theta^{(i,\cdot)[k]})\big)
\end{array}\right. \nonumber
\end{equation}
\end{algorithm*}

\begin{algorithm*}
\small
\caption{-- Strong-PD-$\ell_r$ --  Primal-dual splitting algorithm to solve \eqref{eq:equivform}  when $C = S$} \label{algo:primaldualspecif}
Parameter settings: Set $\beta = 2\sum_{\ell=1}^L \vert \phi(\mathbf{x}_{\ell}) \vert^2 + \sum_{\ell=1}^L \vert \phi(\mathbf{x}_{\ell})^{\top} \phi(\mathbf{x}_{\ell})\vert^2$, $\tau>0$, $\sigma>0$,  {such that $\frac{1}{\tau} - 2\sigma \geq \frac{\beta}{2}$ }.  \\
Initialization: $ (\mathbf{v}^{+[0]}, \mathbf{v}^{-[0]}, \mathbf{\Theta}^{[0]}, \mathbf{s}^{+[0]}, \mathbf{s}^{-[0]}, \mathbf{\Lambda}_1^{[0]}, \mathbf{\Lambda}_2^{[0]})\in \mathbb{R}^N \times\mathbb{R}^N \times\mathbb{R}^{N \times N} \times \mathbb{R}^N \times \mathbb{R}^N \times \mathbb{R}^{N \times N} \times \mathbb{R}^{N \times N} $. \\
For $k=0, 1, \ldots$ 
\begin{equation}
 \left\lfloor \begin{array}{l}  
b_\ell = y_\ell  -   \phi^\top({\bf x}_\ell) (\mathbf{v}^{+[k]} - \mathbf{v}^{-[k]})- \phi^\top({\bf x}_\ell) \mathbf{\Theta}^{[k]} \phi({\bf x}_\ell)\\
\mathbf{v}^{+[k+1]} = P_{O} \big(\mathbf{v}^{+[k]} + \tau \sum_\ell \phi({\bf x}_\ell)  b_\ell - \tau \lambda - \tau \mathbf{s}^{+[k]}  \big) \\
\mathbf{v}^{-[k+1]} = P_{O} \big(\mathbf{v}^{-[k]} - \tau \sum_\ell \phi({\bf x}_\ell)  b_\ell - \tau \lambda  - \tau \mathbf{s}^{-[k]}\big) \\
\mathbf{\Theta}^{[k+1]} = P_{S} \big(\mathbf{\Theta}^{[k]} + \tau \big(\sum_\ell \phi^{(i)}({\bf x}_\ell) \phi^{(j)}({\bf x}_\ell)  b_\ell\big)_{1\leq i\leq N,1\leq j\leq N}  - \tau (\mathbf{\Lambda}_1^{[k]} + \mathbf{\Lambda}_2^{[k]} )\big) \\
\mathbf{\Lambda}_1^{+[k+1]}  = \textrm{prox}_{\frac{\sigma \lambda}{2} \Vert \cdot \Vert_1^{*} } \big( \mathbf{\Lambda}_1^{+[k]}  + \sigma  (2\mathbf{\Theta}^{[k+1]} - \mathbf{\Theta}^{[k]}) \big)\\
\textrm{For $i=1,\ldots,N$ }\\
\lfloor\;\;\; (s^{+(i)[k+1]},s^{-(i)[k+1]},\Lambda_2^{(i,\cdot)[k+1]})  \\
 \qquad\;\; = \textrm{prox}_{\sigma \iota_{E_r}^{*}} \big(  s^{+(i)[k]}  + \sigma  (2{v}^{+[k+1]} - {v}^{+[k]}),  s^{-(i)[k]}  + \sigma  (2{v}^{-(i)[k+1]} - {v}^{-(i)[k]}), \\
 \qquad\;\; \Lambda_2^{+(i,\cdot)[k]}  + \sigma  (2\Theta^{(i,\cdot)[k+1]} - \Theta^{(i,\cdot)[k]})\big)
\end{array}\right. \nonumber
\end{equation}
\end{algorithm*}

\subsection{New epigraphical projection} \label{sec:newepiprjection}
We first focus on the derivation of $P_{E_\infty}$.
\newpage
\begin{proposition} \label{propo:inftyprojection}
Let $\mathbf{u} = (u^{(i)})_{1\leq i \leq N}$. The projection of $(\omega^+, \omega^-, \mathbf{u}) \in \mathbb{R} \times \mathbb{R} \times \mathbb{R}^{N} $ on the epigraphic set $E_\infty$ reads
$$
(\eta^+,\eta^-, \mathbf{p})  =P_{E_\infty}(\omega^+, \omega^-, \mathbf{u}) 
$$
with 
\begin{equation}	
\begin{cases}
 {\eta}^- = \frac{\omega^- - (N - \bar{n}+1)(\omega^+ - \omega^-) + \sum_{i=\bar{n}}^N  \mathbf{\nu}^{(i)}}{1+2(N-\bar{n}+1)}\\
 {\eta}^+ = {\eta}^- + \omega^+ - \omega^- \\
 {\mathbf{p}} = P_{[-{\eta}^+ - {\eta}^-,  {\eta}^+ +{\eta}^-]}(\mathbf{u})
\end{cases}
\end{equation}
\noindent where $(\mathbf{\nu}^{(1)},\ldots,\mathbf{\nu}^{(N)})$ is an ordered version of $(\vert u^{(i)}\vert)_{1\leq i \leq N}$ in an ascending order and with $\mathbf{\nu}^{(0)}=-\infty$, and $\bar{n}\in \{1,\ldots,N\}$ such that $\mathbf{\nu}^{(\bar{n}-1)} < \eta^+ + \eta^- \leq \mathbf{\nu}^{(\bar{n})}$.
\end{proposition}
\begin{proof}
 {The proof is given in Appendix B.}
\end{proof}

The above result is obtained from arguments close to the ones derived in \cite{Chierchia2015} [Proposition 5] for an epigraphical constraint of the form 
$\big\{(\omega,  {\bf u}) \in \mathbb{R}\times\mathbb{R}^N\,\vert \, \max\{ \tau_1\vert u^{(1)}\vert, \ldots,\tau_N\vert u^{(N)}\vert\}\leq \omega\big\}$ where $(\tau_i)_{1\leq i \leq N}$ denotes positive weights. The next proposition is a preliminary result to derive $P_{E_1}$.

\begin{proposition} \label{propo:epil1preli} Let $E_1^+= \{(\omega^+, \omega^-, \mathbf{u}) \in \mathbb{R} \times \mathbb{R} \times \mathbb{R}^{N} \ | \  \Vert \mathbf{u} \Vert_1 = \omega^+ + \omega^-, \mathbf{u} \geq 0 \}$.
The projection of $(\omega^+, \omega^-, \mathbf{u}) \in \mathbb{R} \times \mathbb{R} \times \mathbb{R}^{N} $ on the epigraphic set $E_1^+$ reads
$$
(\eta^+, \eta^-, \mathbf{p})  =P_{E_1^+}(\omega^+, \omega^-, \mathbf{u})
$$
with 
\begin{equation}	
\begin{cases}
\eta^- = \frac{\sum_{i=1}^{\widetilde{n}} {\mu}^{(i)}- \omega^+ + (\widetilde{n}+1) \omega^-}{\widetilde{n}(1+2/\widetilde{n})} \\
\eta^+ = \eta^- + \omega^+ - \omega^- \\
 {\mathbf{p}} =  \mathbf{u}  - \max\Big\{0, \frac{\sum_{i=1}^{\widetilde{n}} {\mu}^{(i)} - ( {\eta}^++  {\eta}^-)}{\widetilde{n}}\Big\}\\
\end{cases}
\end{equation}
with
\begin{equation}
 \widetilde{n} = \max \{n\in\{1,\ldots,N\}\,\vert,\ \mathbf{\mu}^{(n)} - \frac{\sum_{i=1}^n \mathbf{\mu}^{(i)} - (\eta^+ +\eta^-)}{n} >0\}
\label{eq:ntilde}
\end{equation}
and where $\boldsymbol{\mu} = (\mathbf{\mu}^{(i)} )_{1\leq i \leq N}$ is an ordered version of $\mathbf{u} = ( u^{(i)})_{1\leq i \leq N}$ in a descending order.
\end{proposition}
\begin{proof}
 {The proof is given in Appendix C.}
\end{proof}

Next we get the solution of $P_{E_1}$ from the ones of projection to $E_1^+$ according to \cite{Duchi2008}[Lemma 3] and have the following proposition: 
\begin{proposition} \label{prop:epil1}
The projection of $(\omega^+, \omega^-, \mathbf{u}) \in \mathbb{R} \times \mathbb{R} \times \mathbb{R}^{N} $ on the epigraphic set $E_1$ reads
$$
(\eta^+, \eta^-, \mathbf{p})  =P_{E_1}(\omega^+, \omega^-, \mathbf{u}) 
$$
with 
\begin{equation}
 \begin{cases}
  (v^+, v^-,\hat{\mathbf{p}}) = P_{E_1^+}(\omega^+, \omega^-, \abs(\mathbf{u}))\\
 \mathbf{p} = {\sign}(\mathbf{u}) \hat{\mathbf{p}}\\
 \end{cases}
\end{equation}
where $\abs(\cdot)$ and $\sign(\cdot)$ denote the componentwise absolute value and the componentwise sign.
\end{proposition}

Integrating the projection derived in Proposition~\ref{propo:inftyprojection} into Algorithm~\ref{algo:weakprimaldualspecif} or Proposition~\ref{propo:epil1preli} and Proposition~\ref{prop:epil1} into Algorithm~\ref{algo:primaldualspecif} with 
$$
\prox_{\sigma\iota_{E_r}^*}(\omega^+, \omega^-, \mathbf{u}) = (\omega^+, \omega^-, \mathbf{u}) - \sigma P_{E_r}\Big(\frac{\omega^+}{\sigma}, \frac{\omega^-}{\sigma}, \frac{\mathbf{u}}{\sigma}\Big) 
$$
ensures the convergence to a minimizer of \eqref{eq:equivform} respectively for weak hierarchy and strong hierarchy.

\subsection{ {Convergence}} \label{sec:convergence}
 
 The proposed Algorithms~\ref{algo:weakprimaldualspecif} and \ref{algo:primaldualspecif} can be viewed as a particular case of the  primal-dual algorithm~\cite[Algorithm 5.1]{Condat_L_2012} and thus the convergence guarantees can be derived from~\cite[Theorem 5.1]{Condat_L_2012}).

\begin{theorem} \label{theorem:primal-dual}
 {The sequences $(\mathbf{v}^{+[k+1]},\mathbf{v}^{-[k+1]},\mathbf{\Theta}^{[k+1]})_{k\in \mathbb{N}})_{i \in \mathbb{N}}$ generated by Algorithm \ref{algo:weakprimaldualspecif} converges to a solution of \eqref{eq:equivform} when $C = \RR^{N \times N}$ (weak hierarchy).}
\end{theorem}

\begin{theorem} \label{theorem:primal-dual}
 {The sequences $(\mathbf{v}^{+[k+1]},\mathbf{v}^{-[k+1]},\mathbf{\Theta}^{[k+1]})_{k\in \mathbb{N}})_{i \in \mathbb{N}}$ generated by Algorithm \ref{algo:primaldualspecif} converges to a solution of \eqref{eq:equivform} when $C = S$ (strong hierarchy).}
\end{theorem}

 {\begin{remark}
In practice, the choice of $\tau$ and $\sigma$ is made as follows: $\tau=\frac{1.9}{\beta}$ and $\sigma = (1/\tau - \beta/2)$ for Algorithm~\ref{algo:weakprimaldualspecif}  (resp. $\sigma = 1/2*(1/\tau - \beta/2)$  for Algorithm~\ref{algo:primaldualspecif}).
\end{remark}}

\section{Extension to multiclass SVM} \label{sec:multiclassSVM}

\subsection{Model} 

Formally, assuming that the training set is denoted $$\mathcal{T} = \{({y}_\ell, \phi({\bf x}_\ell))\in \{1,\ldots,K\} \times \mathbb{R}^{N}\,\vert\,\ell \in \{1,\ldots,L\}\}$$ with $K$ classes and $L$ training samples. ${\bf x}_\ell\in \mathbb{R}^{M}$ denotes the data (e.g. an image with $M$ pixels) with the label ${y}_\ell$. The transform $\phi\colon \mathbb{R}^{M}\to\mathbb{R}^{N}$ maps from the input space onto an arbitrary
feature space, for instance, scattering features \cite{Bruna_J_2013}. 

Our aim being to study quadratic interactions, for the k-th class, we consider a discrimination function taking the form:
\begin{equation}
f_k(\mathbf{x}_\ell) = \phi(\mathbf{x}_\ell)^{\top} \mathbf{\Theta}_k \phi(\mathbf{x}_\ell) + \mathbf{v}_k^{\top} \phi(\mathbf{x}_\ell)
\end{equation}
where, $\mathbf{v}_k = \big(v^{(i)}_k\big)_{1\leq i \leq N}$  and $ \mathbf{\Theta}_k = \big(\Theta^{(i,j)}_k\big)_{1\leq i \leq N, 1\leq j \leq N}$ model respectively the weights of main features and the matrix of interactions for the class $k\in \{1,\ldots,K\}$. 

\subsection{Design of the objective function} 
A large panel of data term can be encountered in the multiclass SVM literature going from hinge loss to logistic data-term \cite{Blondel2013, Chierchia2015_a}. In this work, we focus on the squared hinge loss data-term proposed in \cite{Blondel2013} leading to good classification performance and being differentiable with Lipschitz gradient $\beta>0$, whose constant will be specified in Section~\ref{ss:pdpep}. The proposed objective function is written as:

\begin{equation} 
 \underset{\substack{\mathbf{v}_k \in \mathbb{R}^N\\ \mathbf{\Theta}_k \in \mathbb{R}^{N\times N}}}{\textrm{minimize}} \!\sum_{\ell=1}^L \sum_{k \neq y_\ell}\max \Big\{ 0, 1 - \big( f_{y_\ell}(\mathbf{x}_\ell) - f_{k}(\mathbf{x}_\ell) \big)\Big\}^2 \!\!+ \sum_{k=1}^K \Omega(\mathbf{v}_k,  \mathbf{\Theta}_k ) \label{equa:generalproblem}
\end{equation}

\noindent
where the first term denotes the squared hinge loss data-term and the second term penalizes the behavior of the discrimination function.
Instead of $\ell_1$-norm, making $\mathbf{v}_k$ and $\mathbf{\Theta}_k$ independent, we can force the hierarchy constraints between $\mathbf{v}_k$ and $\mathbf{\Theta}_k$ for $k$-th class to regulate them. In the classification setting the penalization term is written:

\begin{equation} 
\Omega(\mathbf{v}_k,  \mathbf{\Theta}_k) = \lambda_1 \sum_{i=1}^N \max\{\vert v^{(i)}_k \vert , \Vert \Theta^{(i,\cdot)}_k \Vert_r \} + \lambda_2 \Vert  \mathbf{\Theta}_k  \Vert_1 + \iota_C( \mathbf{\Theta}_k ) 
 \label{equa:hiernetl10}
\end{equation}
where $r = \{1,+\infty\}$, $\Theta^{(i,\cdot)}_k  = \big(\Theta^{(i,j)}_k\big)_{1\leq j \leq N}$ and $\iota_C$ is defined as in Section~\ref{sec:model}.

\subsection{Epigraphical reformulation}

The reformulation of problem~(\ref{equa:generalproblem}) by means of epigraphical constraints \footnote{The epigraph of a function $f\colon \mathcal{H}\to ]-\infty,+\infty]$ is defined as: $\mathrm{epi} f = \{ (v,\zeta)\in \mathcal{H} \times \mathbb{R}\,\vert\, f(v)\leq \zeta\}$ and the epigraphical projection onto $\mathrm{epi} f$ is denoted $P_{\textrm{epi} f}$.} follows ideas derived in \cite{Bien2013} (i.e.~``hierNet'') in the context of regression and for a different penalization $\Omega$. 

\begin{proposition} \label{prop:episvm}The minimization problem~\eqref{equa:generalproblem}

with $\Omega(\mathbf{v}_k, \Theta_k)$ defined in \eqref{equa:hiernetl10} can be reformulated as
{\small
\begin{multline}
\label{eq:equivformclass}
 \underset{(\mathbf{v}^+_k, \mathbf{v}^-_k,\Theta_k)\in O\times O \times C}{\textrm{minimize}} \sum_{\ell=1}^L \sum_{k \neq y_\ell}\max \Big\{ 0, 1 - \big( f_{y_\ell}(\mathbf{x}_\ell) - f_{k}(\mathbf{x}_\ell) \big)\Big\}^2  \\ \hspace{0.75cm} + \sum_{k=1}^K \Big( \lambda_1 1^\top(\mathbf{v}^+_k +\mathbf{v}^-_k) + \lambda_2 \sum_{i=1}^N  \Vert  \Theta^{(i,\cdot)}_k \Vert_1 \\ +  \sum_{i=1}^N \iota_{E_r}\big(v^{+(i)}_k,v^{-(i)}_k, \Theta^{(i,\cdot)}_k\big) + \iota_C( \mathbf{\Theta}_k)  \Big)
\end{multline} }
where $O$ denotes the positive orthant and ${E_r}= \{(\omega^+,\omega^-,  {\bf u}) \in \mathbb{R}\times\mathbb{R}\times\mathbb{R}^N\,\vert \, \Vert {\bf u} \Vert_r \leq \omega^++\omega^-\}$ with  $r = \{1,+\infty\}$ and $f_k(\mathbf{x}_\ell) = \phi(\mathbf{x}_\ell)^{\top} \mathbf{\Theta}_k \phi(\mathbf{x}_\ell) + (\mathbf{v}_k^+ - \mathbf{v}_k^-)^{\top} \phi(\mathbf{x}_\ell)$.
\end{proposition}
\noindent The variable $\mathbf{v}_k$ is written as $\mathbf{v}_k=\mathbf{v}^{+}_k-\mathbf{v}^{-}_k$ by variable splitting. The proof relies on similar arguments than for Proposition~\ref{prop:epi}.

\subsection{Primal-dual proximal algorithm based on epigraphical projection} 
\label{ss:pdpep}
The proposed objective function~(\ref{eq:equivform}) is a specific case of \eqref{eq:primal} when $\mathbf{w} = (\mathbf{v}^+_k,\mathbf{v}^-_k, {\bf{\Theta}}_k) \in \mathbb{R}^N \times\mathbb{R}^N \times \mathbb{R}^{N\times N}$, $Q=2$ and
\begin{equation*} \label{equa:generaldecompo}
\small
\begin{cases}
f(\mathbf{w}) = \sum_{\ell=1}^L \sum_{k \neq y_\ell}\max \Big\{ 0, 1 - \big( f_{y_\ell}(\mathbf{x}_\ell) - f_{k}(\mathbf{x}_\ell) \big)\Big\}^2 \\
\quad \quad  \quad \quad  + \sum_{k=1}^K \lambda_1 1^\top(\mathbf{v}^+_k + \mathbf{v}^-_k)\\
g(\mathbf{w}) = \sum_{k=1}^K \iota_{O}(\mathbf{v}^+_k)  +\iota_{O}(\mathbf{v}^-_k) + \iota_{S}(\Theta_k) \\
h_1(\mathbf{w}) = \sum_{k=1}^K \lambda_2 \Vert \mathbf{\Theta}_k \Vert_1 \\
h_2(\mathbf{w}) = \sum_{k=1}^K \sum_{i=1}^N \iota_{E_r}\big(v^{+(i)}_k,v^{-(i)}_k, \Theta^{(i,\cdot)}_k\big) .
\end{cases}
\end{equation*}
where $f$ is differentiable with a Lipschitz constant $\beta=4(K-1)(\sum_{\ell} \Vert \phi(\mathbf{x}_\ell) \Vert^2 + \sum_{\ell} \Vert \phi(\mathbf{x}_\ell)\phi(\mathbf{x}_\ell )^{\top} \Vert^2)$.
The primal-dual proximal iterations are displayed in Algorithm~\ref{algo:primaldualspecif}. It involves four proximity operator computation: the projection onto $O$ and $S$, the proximal operator of $\ell_1$-norm and the epigraphical projection onto $E_r$, whose closed form expression are derived in Section~\ref{sec:frdual}.

\begin{algorithm*}
\small
\caption{Primal-dual proximal algorithm based on epigraphical projection for classification when $C = S$} \label{algo:primaldualspecif}
Parameter settings: Set $\tau>0$, $\sigma>0$, such that $2(1/\tau - \sigma)\geq \beta$. \\
Initialization: $ (\mathbf{v}^{+[0]}_k, \mathbf{v}^{-[0]}_k, \mathbf{\Theta}^{[0]}_k, \mathbf{s}^{+[0]}_k, \mathbf{s}^{-[0]}_k, \mathbf{\Lambda}_{k,1}^{[0]}, \mathbf{\Lambda}_{k,2}^{[0]}) \in \mathbb{R}^N \times\mathbb{R}^N \times\mathbb{R}^{N \times N} \times \mathbb{R}^N \times \mathbb{R}^N \times \mathbb{R}^{N \times N} \times \mathbb{R}^{N \times N} \quad \forall k=\{1,\ldots, K\}$. \\
For $t=0, 1, \ldots$ 
\begin{equation}
 \left\lfloor \begin{array}{l}  

f(\mathbf{v}^+, \mathbf{v}^-, \mathbf{\Theta}) = \sum_{\ell=1}^L \sum_{k \neq y_\ell}\max \Big\{ 0, 1 - \big( \phi(\mathbf{x}_\ell)^{\top} \mathbf{\Theta}_{y_\ell} \phi(\mathbf{x}_\ell) + (\mathbf{v}_{y_\ell}^+ - \mathbf{v}_{y_\ell}^-)^{\top} \phi(\mathbf{x}_\ell) -  \\
\quad\quad\quad\quad\quad\quad\quad\quad\phi(\mathbf{x}_\ell)^{\top} \mathbf{\Theta}_{k} \phi(\mathbf{x}_\ell) - (\mathbf{v}_{k}^+ - \mathbf{v}_{k}^-)^{\top} \phi(\mathbf{x}_\ell) \big)\Big\}^2+ \sum_{k=1}^K \lambda_1 1^\top(\mathbf{v}^+_k + \mathbf{v}^-_k); \\

\mathbf{v}^{+[t+1]}_k = P_{O} \big(\mathbf{v}^{+[t]}_k - \tau \nabla_{\mathbf{v}_k^+} f(\mathbf{v}^{+[t]}_k, \mathbf{v}^{-[t]}_k, \mathbf{\Theta}^{[t]}_k)  - \tau \mathbf{s}^{+[t]}_k  \big) \quad \forall k=1,\ldots,K \\
\mathbf{v}^{-[t+1]}_k = P_{O} \big(\mathbf{v}^{-[t]}_k - \tau \nabla_{\mathbf{v}_k^-} f(\mathbf{v}^{+[t]}_k, \mathbf{v}^{-[t]}_k, \mathbf{\Theta}^{[t]}_k)  - \tau \mathbf{s}^{-[t]}_k\big) \quad \forall k=1,\ldots,K \\
\mathbf{\Theta}^{[t+1]}_k = P_{S} \big(\mathbf{\Theta}^{[t]}_k - \tau \nabla_{\Theta_k} f(\mathbf{v}^{+[t]}_k, \mathbf{v}^{-[t]}_k, \mathbf{\Theta}^{[t]}_k)  - \tau (\mathbf{\Lambda}_{k,1}^{[t]} + \mathbf{\Lambda}_{k,2}^{[t]} )\big) \quad \forall k=1,\ldots,K \\

\mathbf{\Lambda}_{k,1}^{+[t+1]}  = \textrm{prox}_{\sigma \lambda_2 \Vert \cdot \Vert_1^{*}} \big( \mathbf{\Lambda}_{k,1}^{+[t]}  + \sigma  (2\mathbf{\Theta}^{[t+1]}_k - \mathbf{\Theta}^{[t]}_k) \big) \quad \forall k=1,\ldots,K \\

\textrm{For $k=1,\ldots,K$, \quad $i=1,\ldots,N$ } \\
\lfloor\;\;\; (s^{+(i)[t+1]}_k,s^{-(i)[t+1]}_k,\Lambda_{k,2}^{(i,\cdot)[t+1]})  \\
 \qquad\;\; = \textrm{prox}_{\sigma \iota_{E_r}^{*}} \big(  s^{+(i)[t]}_k  + \sigma  (2{v}^{+[t+1]}_k - {v}^{+[t]})_k,  s^{-(i)[t]}_k  + \sigma  (2{v}^{-(i)[t+1]}_k - {v}^{-(i)[t]}_k), \\
\quad\quad\quad\quad\quad\quad \quad\quad\quad\quad\quad\quad\quad\quad\quad\quad\quad\quad\quad\quad\quad \Lambda_{k,2}^{+(i,\cdot)[t]}  + \sigma  (2\Theta^{(i,\cdot)[t+1]}_k - \Theta^{(i,\cdot)[t]}_k)\big)
\end{array}\right. \nonumber
\end{equation}
\end{algorithm*}

\section{Experiments} \label{sec:experiments}
In this section, we firstly provide a comparison between $\ell_1$ and $\ell_\infty$ proposed approaches with the two closest state-of-the-art procedures, that are ``hierNet'' \cite{Bien2013} and Jenatton's framework~\cite{Jenatton2011b}, on a simulation dataset into regression framework. Comparisons are of two types, in terms of convergence behavior and in term of performing estimation.  

Second, we apply the regression algorithms to HIV disease analysis, which is meaningful to the prevention of the disease. Third, classification experiments are conducted in the context of Parkinson disease classification and face classification. The proposed algorithms and the comparison approaches are implemented in Matlab on a computer with AMD AthlonX4 750 processor.

\subsection{Simulated data}

\subsubsection{Dataset} \label{sec:similutaiondata}
The dataset is created according to \cite{Haris2014}. It is initially composed of $N$ main features and of $N(N-1)/2$ interactions (due to symmetry). We denote ($\overline{\mathbf{v}},\overline{\boldsymbol{\Theta}})\in \mathbb{R}^N \times \mathbb{R}^{N\times N}$ the ground truth generated according to strong hierarchy. $\overline{\mathbf{v}}$ denotes a sparse vector where the non-zero values are associated with randomly selected indexes. The value for the non-zero $\overline{v}^{(i)}$ is randomly selected from  the set $\{-5, -4, \ldots, -1, 1, \ldots, 5\}$. Due to strong hierarchical constraint, $\overline{\boldsymbol{\Theta}}$ is only non-zero for those whose main effects are not zero. The values are randomly chosen from the set $\{-10, -8, \ldots, -2, 2, \ldots, 8,10\}$. 

The algorithmic performance is evaluated on two simulated datasets. The first one is composed of $N=30$ features, the first 10 main features are non-zeros and the interaction ratio is $\rho=3.45\%$ (Dataset30-345). The second dataset is of size $N=100$, the first 30 features are non-zeros and the interaction ratio is  $\rho=0.30\%$ (Dataset100-030).  {The interaction ratio is calculated as the non-zero interaction number divided by the possible interaction number in $\overline{\boldsymbol{\Theta}}$ (i.e.~$N(N-1)/2$).}

The datasets are composed of a training, a validation and a testing set with 100 samples each. $\phi(\mathbf{x}_\ell)$ is randomly generated according to normal distribution $\mathcal{N}(0, \mathbb{I}_N)$. The observed value for each sample is set by $\mathbf{y}_\ell = \phi(\mathbf{x})_\ell^\top \overline{\mathbf{v}} + \phi(\mathbf{x}_\ell)^\top \overline{\boldsymbol{\Theta}} \phi(\mathbf{x}_\ell) + \varepsilon_\ell$, where $\varepsilon_\ell$ is independent Gaussian noise to make the signal-to-noise ratio approximately equal to $5$dB. 

\subsubsection{Comparison with \cite{Jenatton2011b} -- $r=\infty$ and weak hierarchy} In order to provide fair comparison with the work in \cite{Jenatton2011b}, we first investigate the performance of proposed primal-dual proximal algorithm (Algorithm~\ref{algo:weakprimaldualspecif}) when $r=\infty$ and in the configuration of weak hierarchy (no symmetry constraint is considered, i.e. $C=\mathbb{R}^{N\times N}$): ``Weak-PD-$\ell_\infty$''. The algorithmic procedure designed in \cite{Jenatton2011b}  is based on ``FISTA'' and it is named ``Weak-FISTA-$\ell_\infty$'',  {where we use the proximity operator of tree-structured variables in ``SPAMS'' toolbox \cite{Jenatton2011b}, it is worthy mentioning that the core function is written in C++, while the whole implementation of our proposed epigraphical primal-dual algorithm is implemented in MATLAB.} Both are compared in two respects: the convergence of the objective function Eq.~\eqref{eq:equivform} and convergence of the iterates (i.e. $\Vert \mathbf{w}^{[k]} - \mathbf{w}^{[+\infty]} \Vert$). 

The comparisons are displayed in Fig.~\ref{fig:weakinftycompar30} (for Dataset30-345) and in Fig.~\ref{fig:weakinftycompar100} (for Dataset100-030)  {for the optimal $\lambda$ (cf. Section~\ref{sec:similperf})}. The convergence behaviour are presented both w.r.t. iterations (first row) and time (second row). It is observed that: 
\begin{enumerate}
\item[i)] ``Weak-PD-$\ell_\infty$'' and ``Weak-FISTA-$\ell_\infty$'' converge to the same value as shown in the first column of both Fig.~\ref{fig:weakinftycompar30} and \ref{fig:weakinftycompar100}; 
\item[ii)] From the objective function convergence point-of-view, ``Weak-FISTA-$\ell_\infty$'' converges faster than `Weak-PD-$\ell_\infty$'' w.r.t. iteration numbers and also time;
\item[iii)] From the convergence of the iterates point of view (second column of Fig.~\ref{fig:weakinftycompar30} and~\ref{fig:weakinftycompar100}),  {for Dataset30-345, ``Weak-PD-$\ell_\infty$'' converges much faster than ``Weak-FISTA-$\ell_\infty$'' to the optimal solution, especially from the comparison in terms of time (bottom figures). For Dataset100-030, although it seems that the convergence in iteration for ``Weak-PD-$\ell_\infty$'' slower, but it costs less time, this may because that the default values for $\tau$ and $\sigma$ are not optimal for Dataset100-030.} It shows that the proposed algorithm enables to be closest to the optimal solution with less time;
\item[iv)]  {Contrary to``Weak-FISTA-$\ell_\infty$'', our proposed solution can provide a `Strong-PD-$\ell_\infty$'' counterpart of this `Weak-PD-$\ell_\infty$'' without inner iterations.}
\end{enumerate}
 {Further results, especially w.r.t the choice of $\lambda$ and regression performance on the training/validation/test sets are provided in Section~\ref{sec:similperf}}.

\begin{figure}[h]
\centering
	\includegraphics[height=3.5cm]{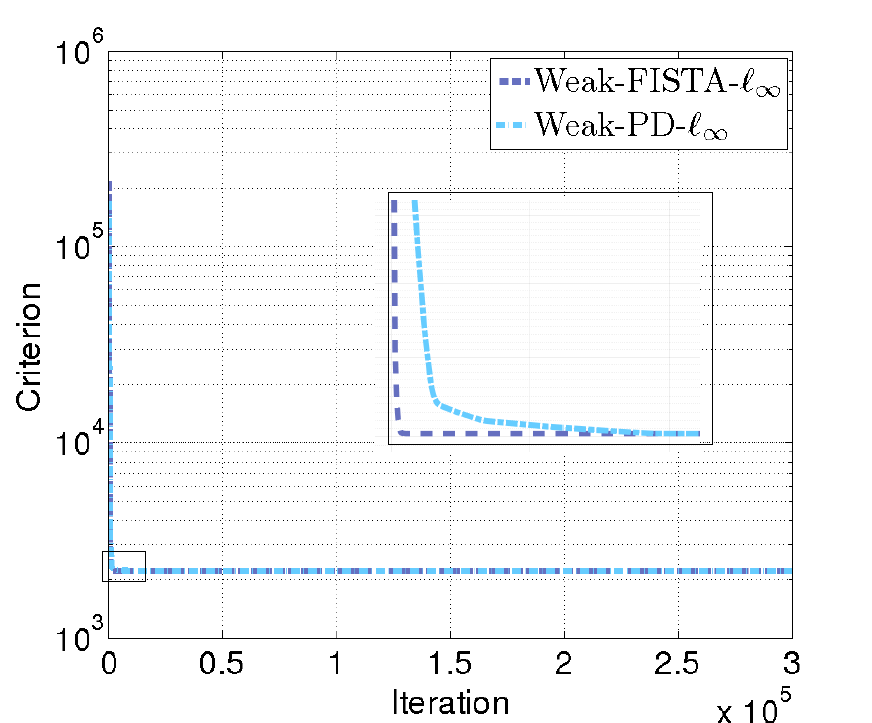} 
	\includegraphics[height=3.5cm]{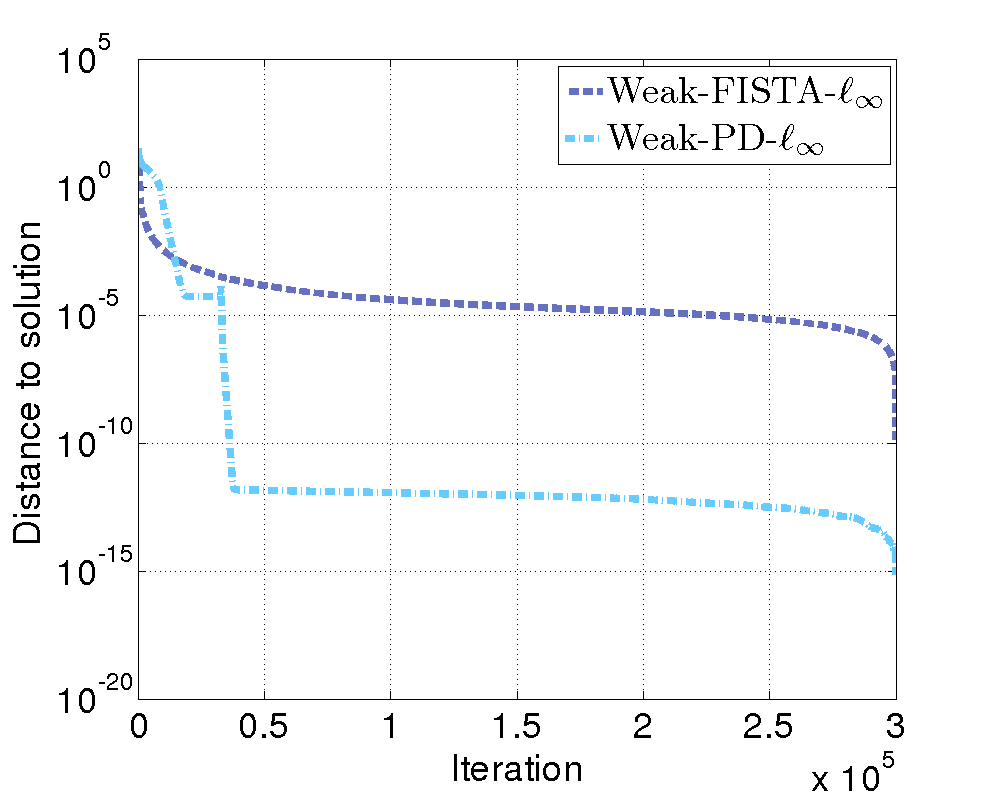}	
	\\
	\includegraphics[height=3.5cm]{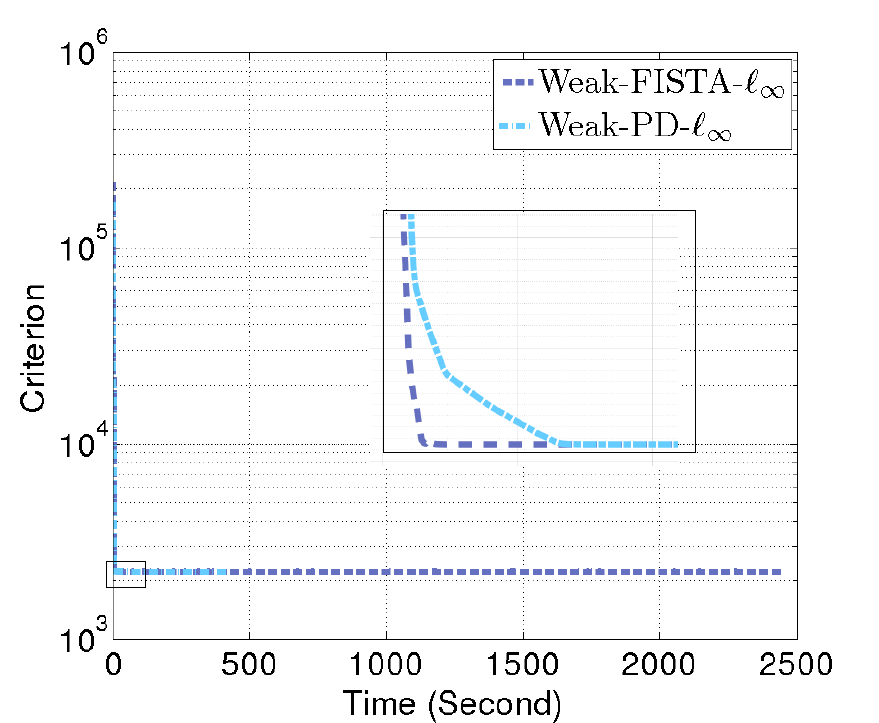}
	\includegraphics[height=3.5cm]{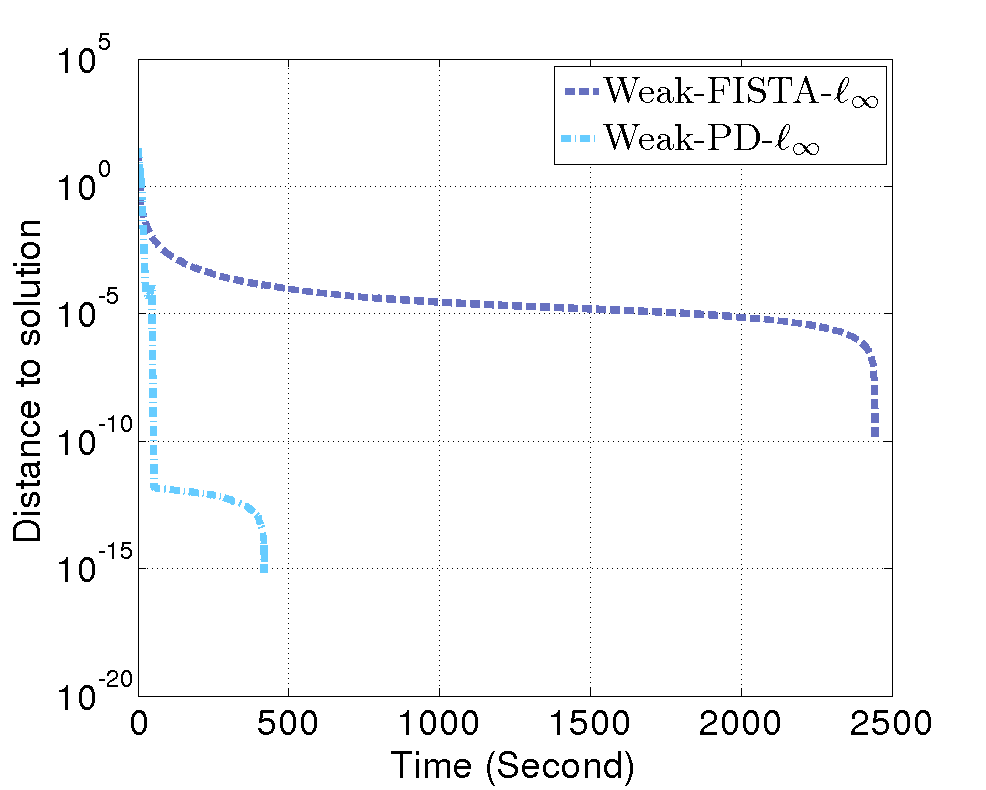}
	\caption{Comparison between the proposed ``Weak-PD-$\ell_\infty$''  and ``Weak-FISTA-$\ell_\infty$'' on Dataset30-345 for $\lambda=14$. (top-left) Objective function in \eqref{eq:regform} w.r.t. iterations, (bottom-left) Objective function in w.r.t. time, (top-right) Distance $\Vert\mathbf{w}^{[k]}-\mathbf{w}^{[\infty]} \Vert$ w.r.t. iterations, (bottom-right) Distance $\Vert\mathbf{w}^{[k]}-\mathbf{w}^{[\infty]} \Vert$ w.r.t. time.} \label{fig:weakinftycompar30}
\end{figure}

\begin{figure}[h]
\centering
	\includegraphics[height=3.5cm]{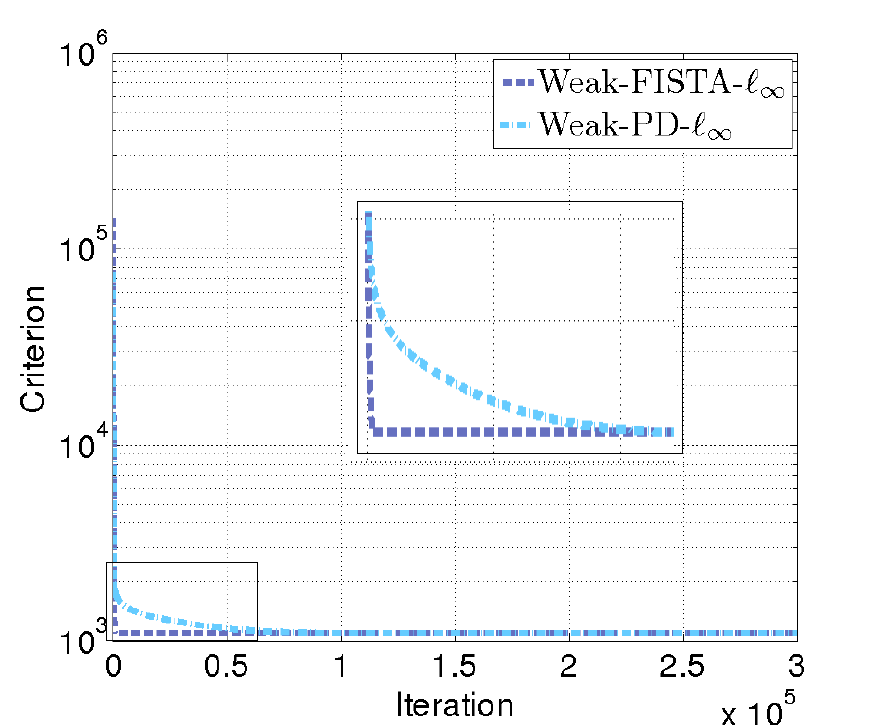} 
	\includegraphics[height=3.5cm]{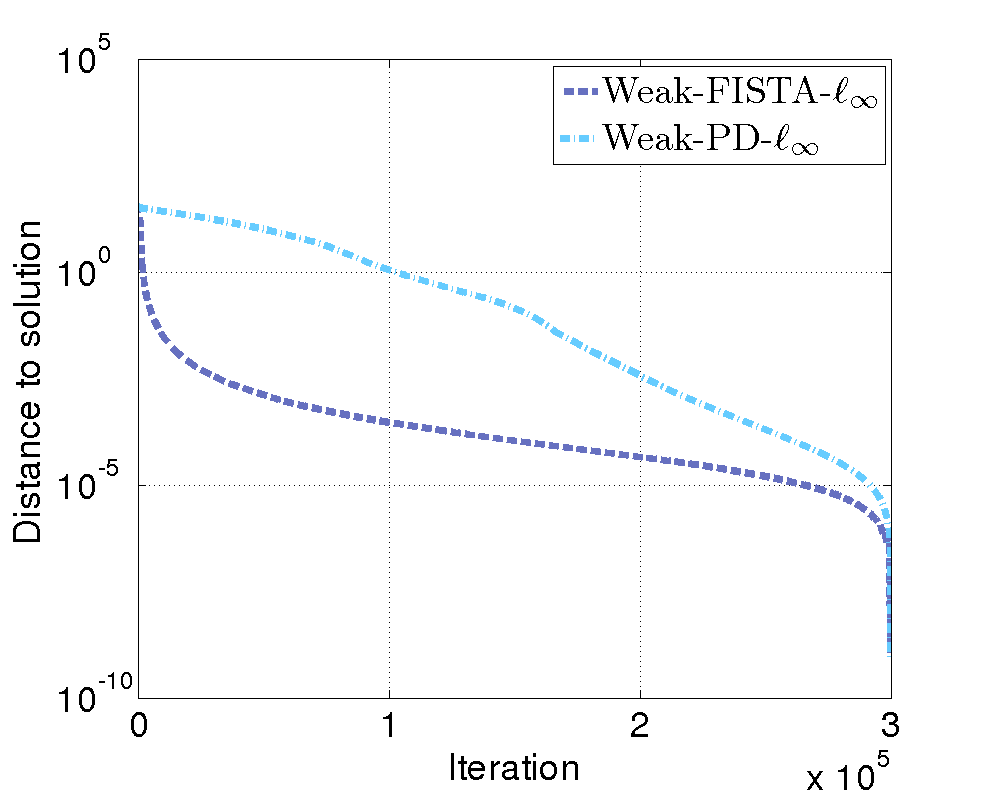} 
	\\
	\includegraphics[height=3.5cm]{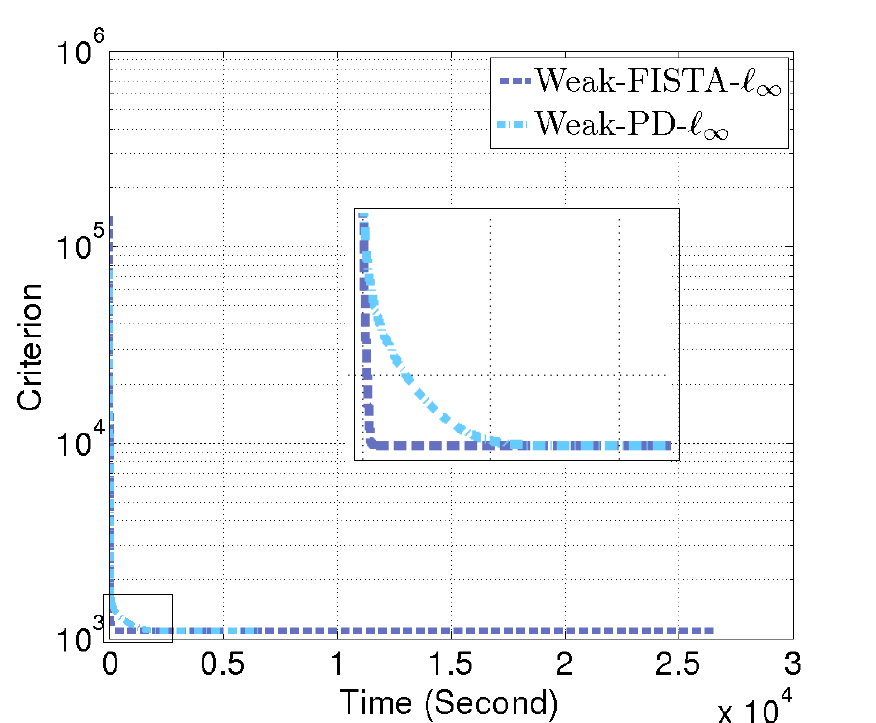}
	\includegraphics[height=3.5cm]{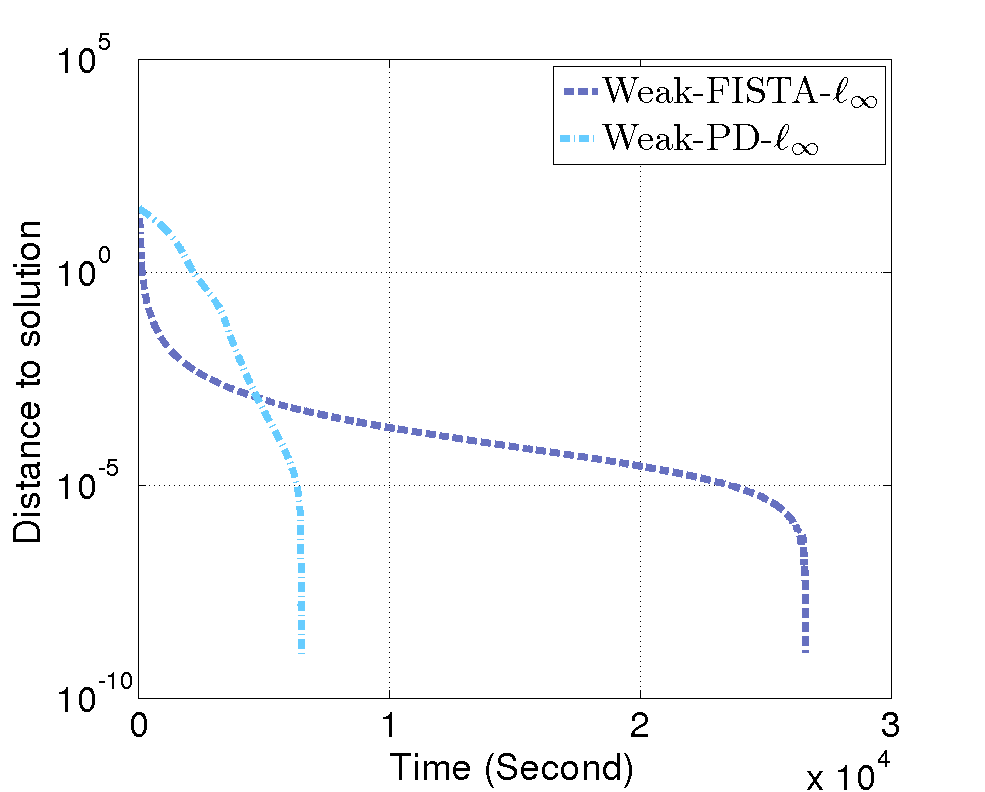} 	
	\caption{Comparison between the proposed ``Weak-PD-$\ell_\infty$''  and ``Weak-FISTA-$\ell_\infty$ on Dataset100-030 for $\lambda=8$. (top-left) Objective function \eqref{eq:regform}  w.r.t. iterations, (bottom-left) Objective function \eqref{eq:regform}  w.r.t. time, (top-right) Distance $\Vert\mathbf{w}^{[k]}-\mathbf{w}^{[\infty]} \Vert$ w.r.t. iterations, (bottom-right) Distance $\Vert\mathbf{w}^{[k]}-\mathbf{w}^{[\infty]} \Vert$ w.r.t. time.}
\label{fig:weakinftycompar100}
\end{figure}

\subsubsection{Comparison with \cite{Bien2013} -- $r=1$ and strong hierarchy}  Similar experiments are conducted  for $r=1$ and strong hierarchy  (i.e., $C=S$). The proposed algorithm (Algorithm~\ref{algo:primaldualspecif} with $r=1$) is called ``Strong-PD-$\ell_1$" and it is compared with ``hierNet'' whose iterations are derived from an ADMM scheme, named``Strong-ADMM-$\ell_1$",  {it is noted that the hyperparameter in the inner iteration for ADMM is set to be the default value (i.e.~1) in the following experiments.}

The comparisons between these two schemes are displayed in Fig.~\ref{fig:strongl1compar30} (for Dataset30-345) and Fig.~\ref{fig:strongl1compar100-01} (for Dataset100-030)  {for the optimal $\lambda$ (cf. Section~\ref{sec:similperf})}. Both are compared in three respects: the convergence of the objective function Eq.~\eqref{eq:regform}, the convergence of the iterates (i.e. $\Vert \mathbf{w}^{[k]} - \mathbf{w}^{[+\infty]} \Vert$), and the distance to the set $S$. These quantities are displayed w.r.t. iteration number and time. It is observed that: 
\begin{enumerate}
\item[i)] ``Strong-PD-$\ell_1$" and ``Strong-ADMM-$\ell_1$" converge to the same solution, but the convergence of the objective to the optimum solution is very sensitive  to the trade-off hyper-parameter for ``Strong-ADMM-$\ell_1$";
\item[ii)]  ``Strong-PD-$\ell_1$" is always faster either in iteration number or in time and either in terms of functional or in terms of iterations. The explanation mainly comes from the inner iterations required with ``Strong-ADMM-$\ell_1$"  when dealing with $C=S$; 
\item[iii)] Third column of Fig.~\ref{fig:strongl1compar30} and ~\ref{fig:strongl1compar100-01} highlights that the constraint violations (i.e. distance to $S$) with the proposed method is always smaller than with ADMM;
\end{enumerate}
 {Further results, especially w.r.t the choice of $\lambda$ and regression performance on the training/validation/test sets are provided in Section~\ref{sec:similperf}.}

\begin{figure}[H]
\centering
	\includegraphics[height=3cm]{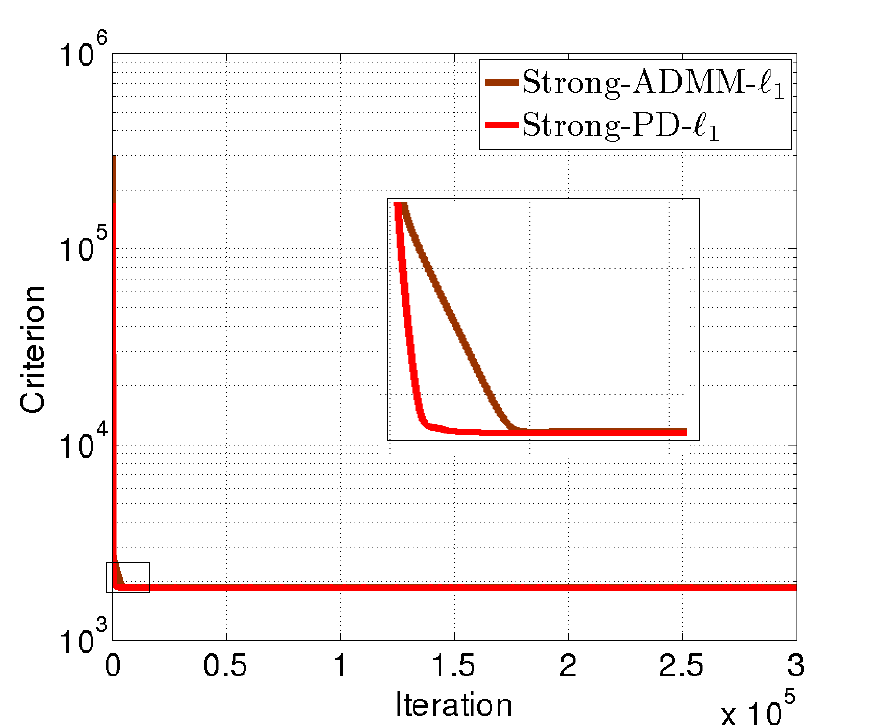} 
	\includegraphics[height=3cm]{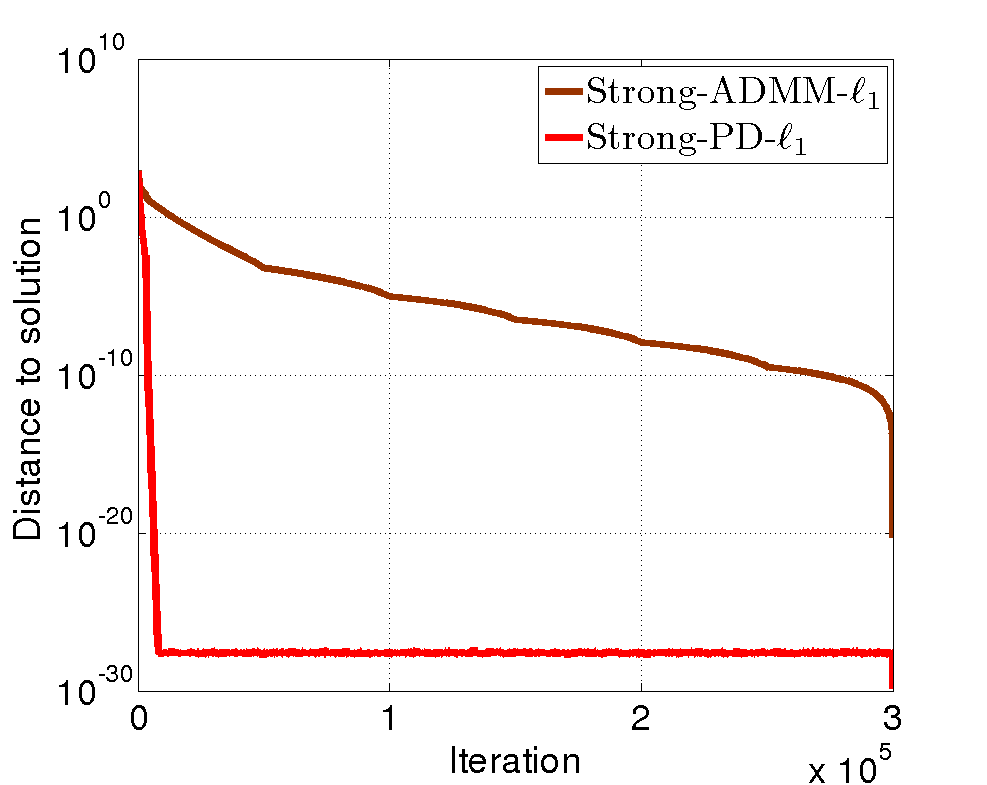} 
	\includegraphics[height=3cm]{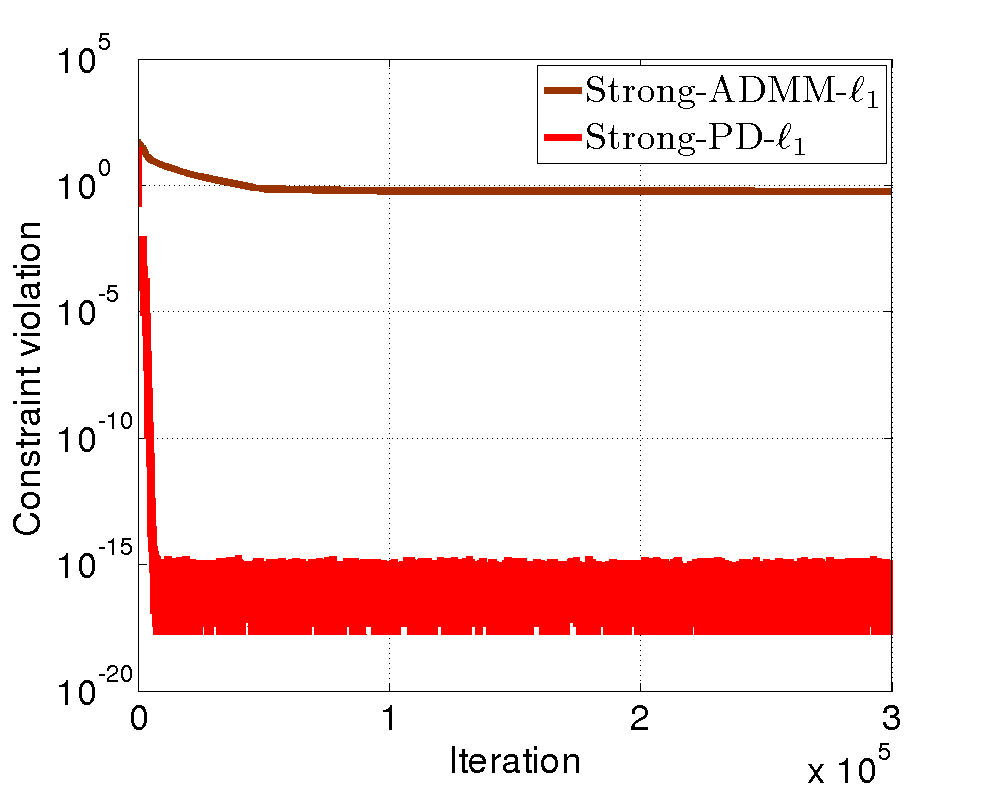} 
	\\	
	\includegraphics[height=3cm]{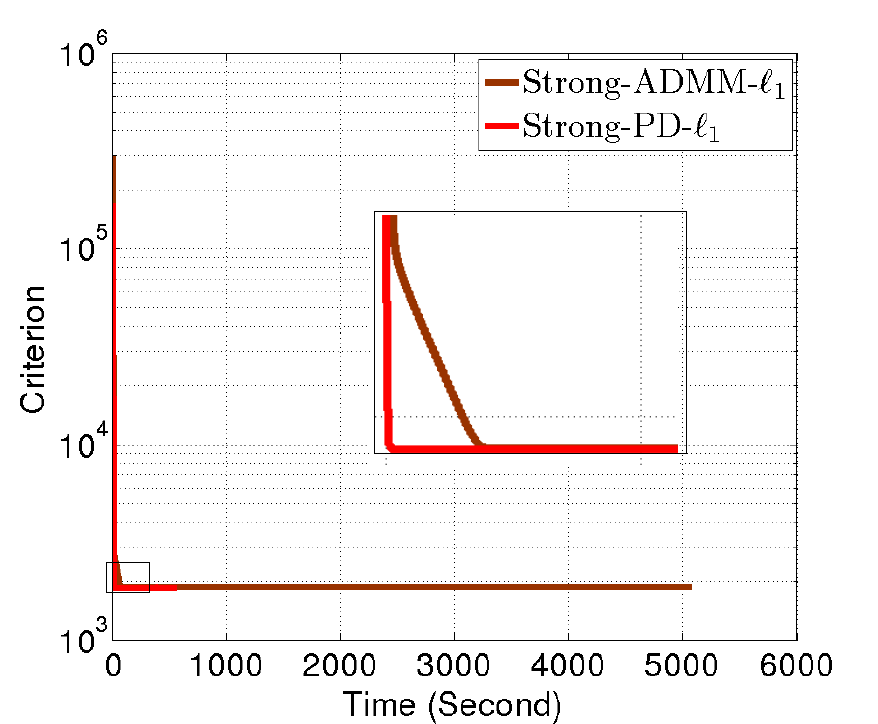}
	\includegraphics[height=3cm]{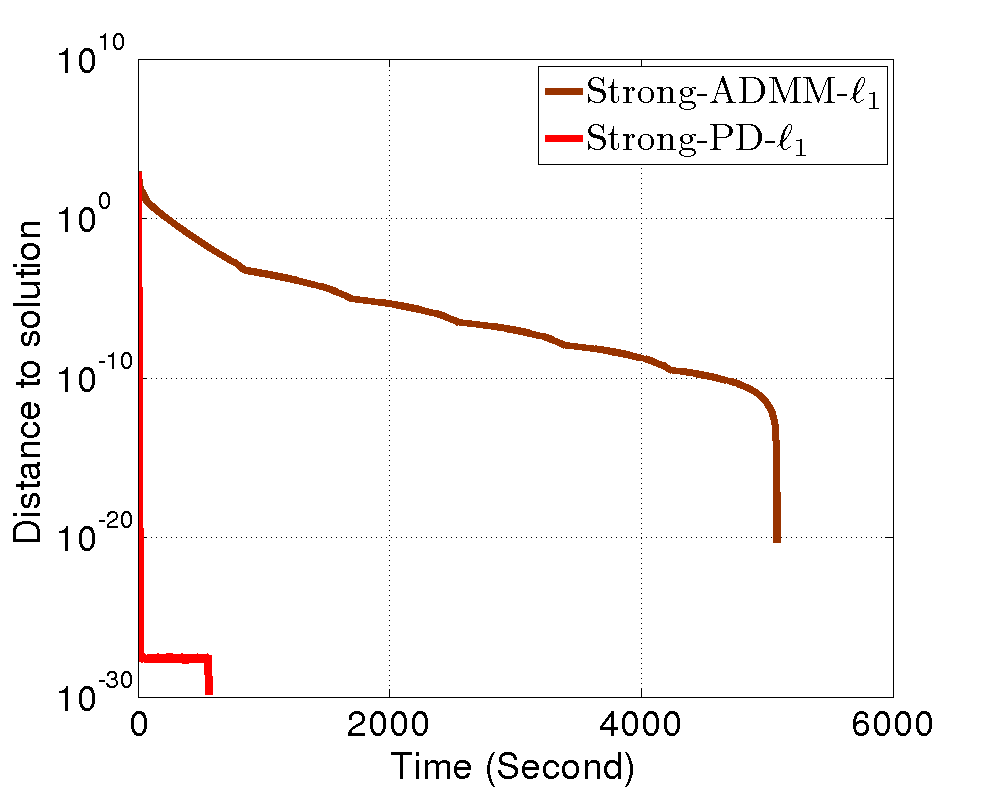}	
	\includegraphics[height=3cm]{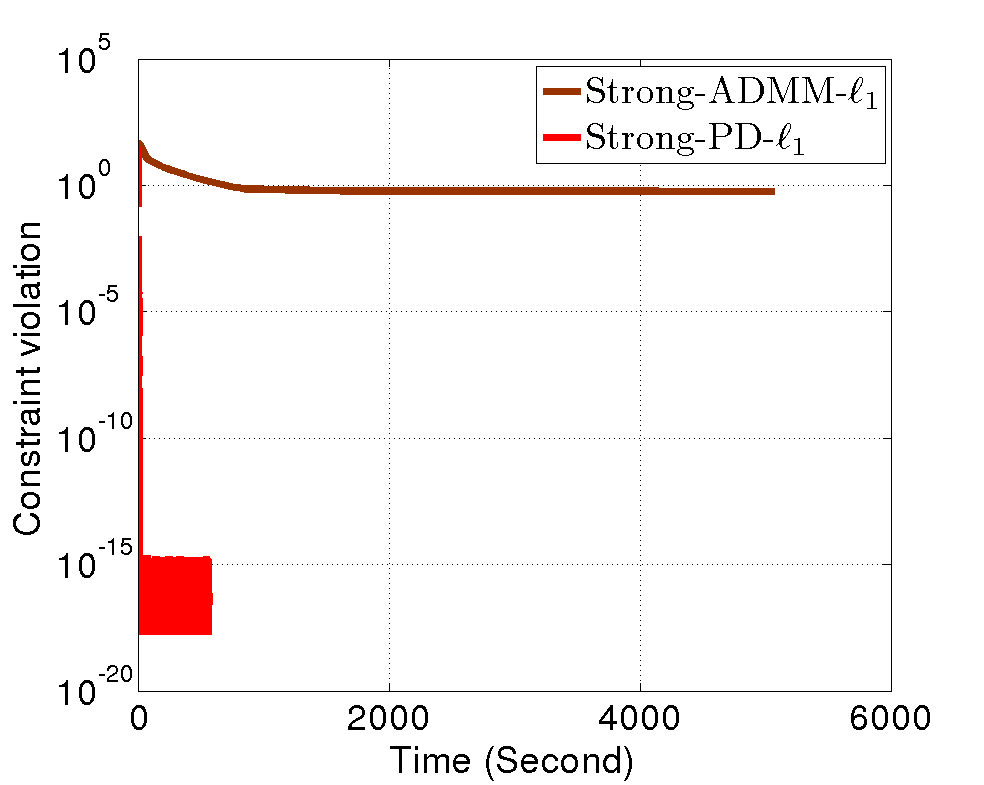}	
	\caption{Comparison between the proposed ``Strong-PD-$\ell_1$''  and ``Strong-ADMM-$\ell_1$'' on Dataset30-345 for $\lambda=8$. (top-left) Objective function \eqref{eq:regform}  w.r.t. iterations, (bottom-left) Objective function \eqref{eq:regform}  w.r.t. time, (top-middle) Distance $\Vert\mathbf{w}^{[k]}-\mathbf{w}^{[\infty]} \Vert$ w.r.t. iterations, (bottom-middle) Distance $\Vert\mathbf{w}^{[k]}-\mathbf{w}^{[\infty]} \Vert$ w.r.t. time, (top-right) Distance to the strong hierarchy constraint $S$ w.r.t. iterations, (bottom-right) Distance to the strong hierarchy constraint $S$ w.r.t. time. \label{fig:strongl1compar30}}
\end{figure}

\begin{figure}[H]
\centering
	\includegraphics[height=3cm]{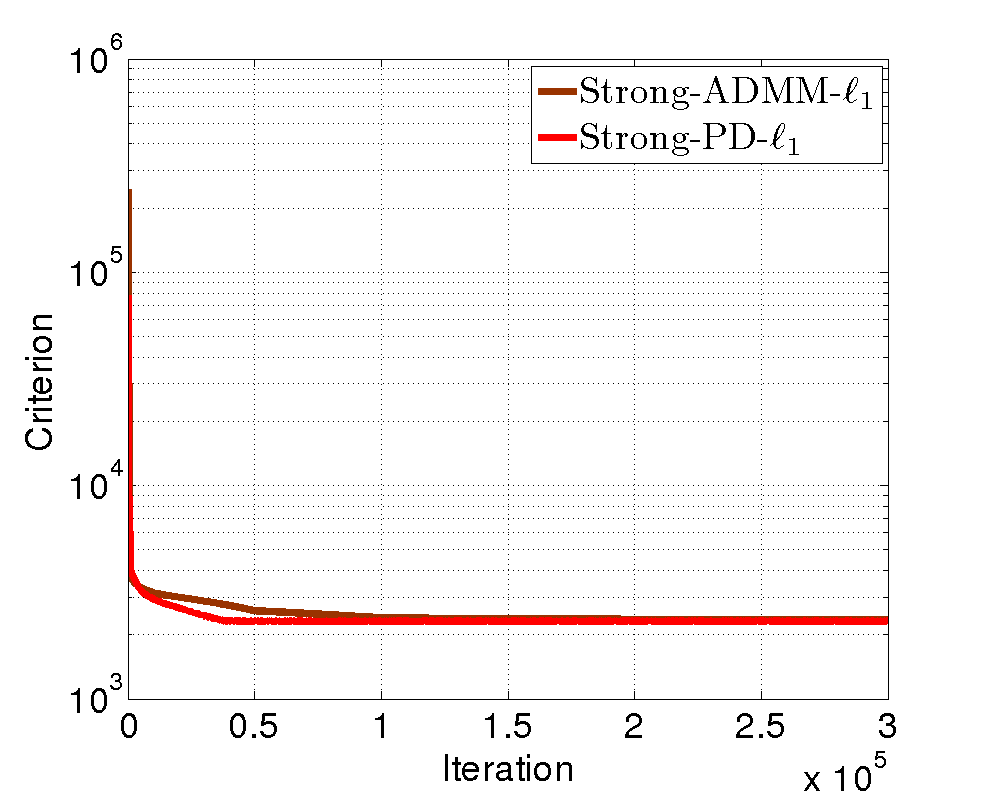} 
	\includegraphics[height=3cm]{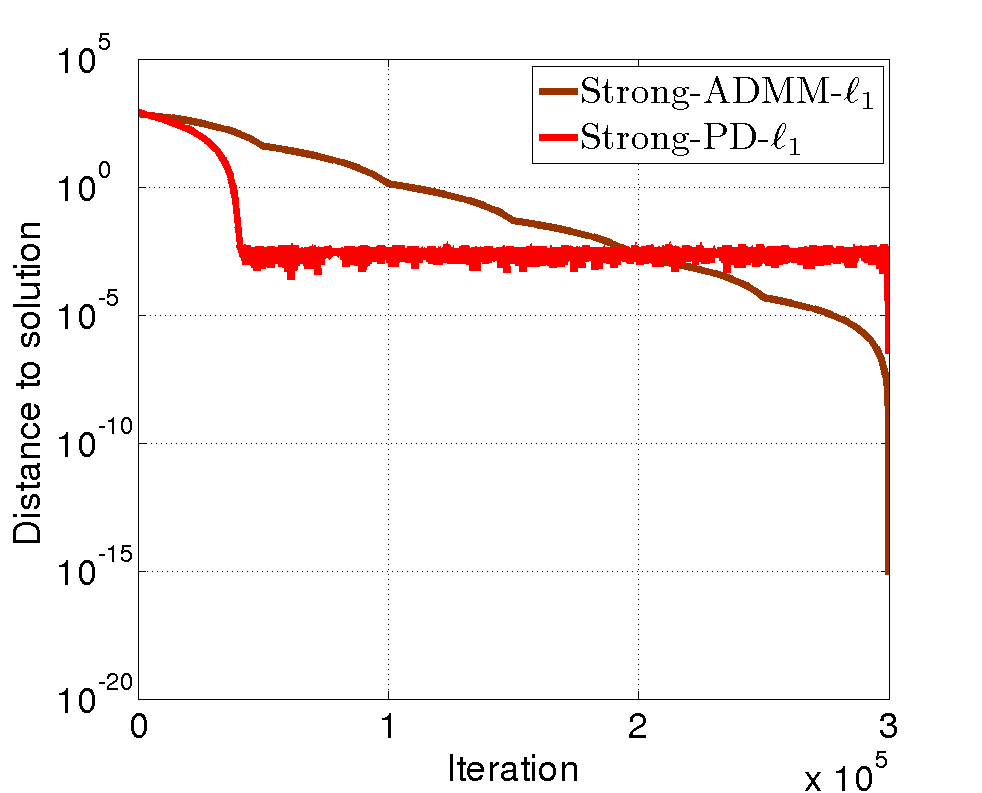} 
	\includegraphics[height=3cm]{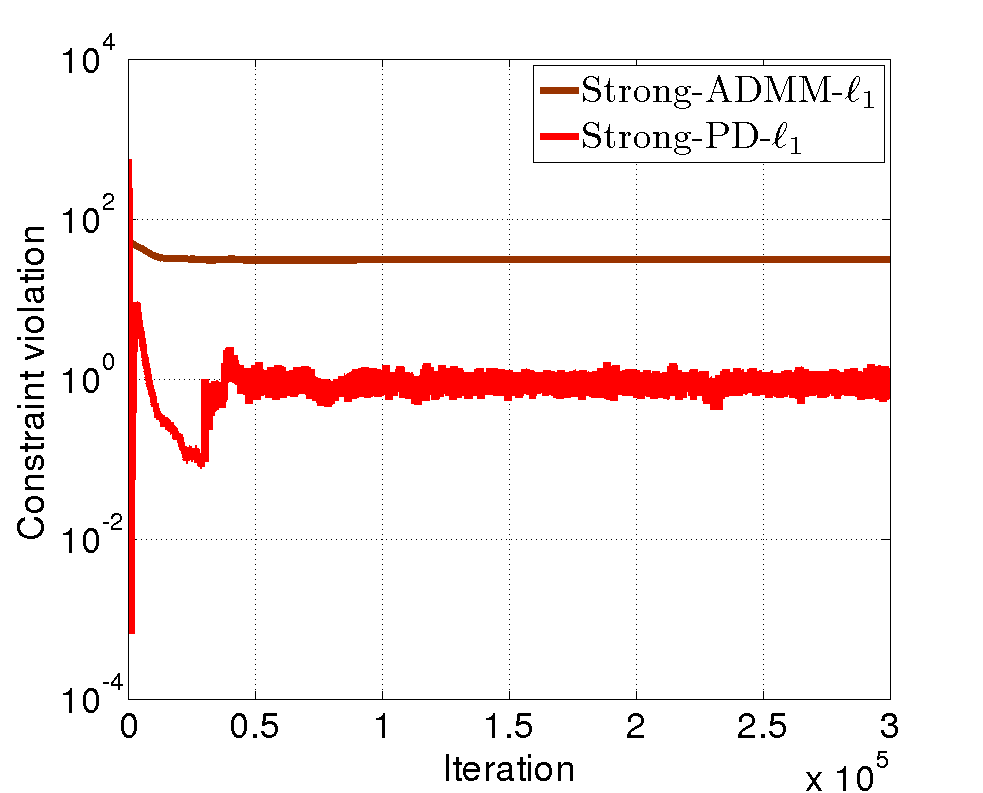} 	
	\\	
	\includegraphics[height=3cm]{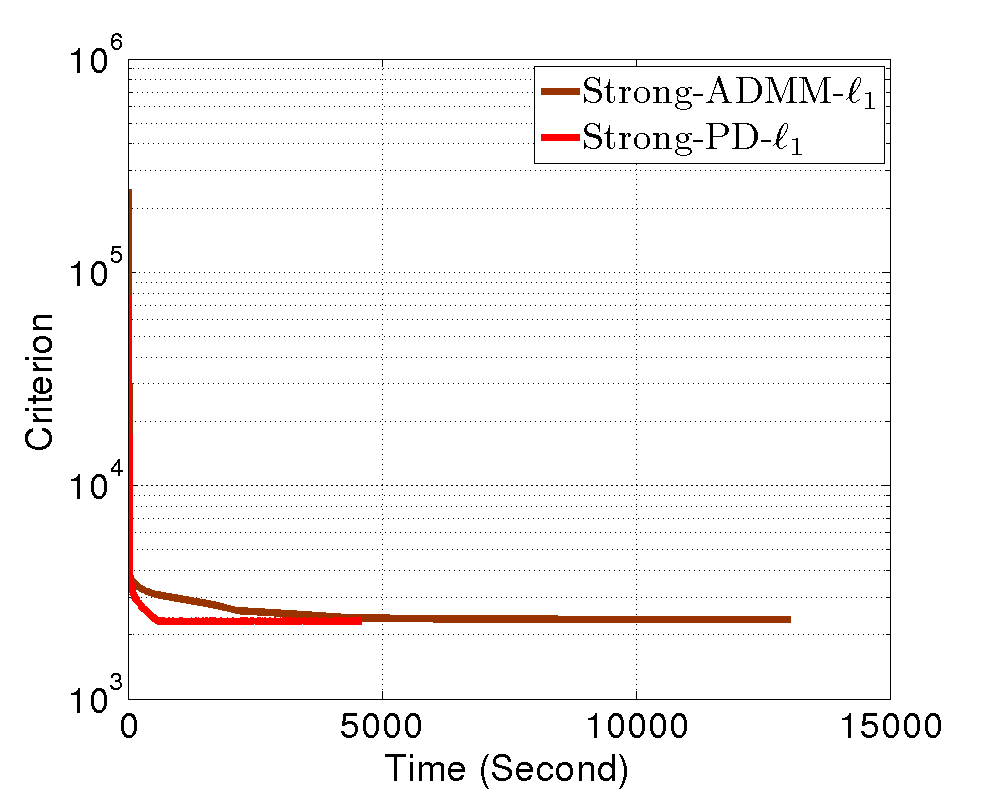}
	\includegraphics[height=3cm]{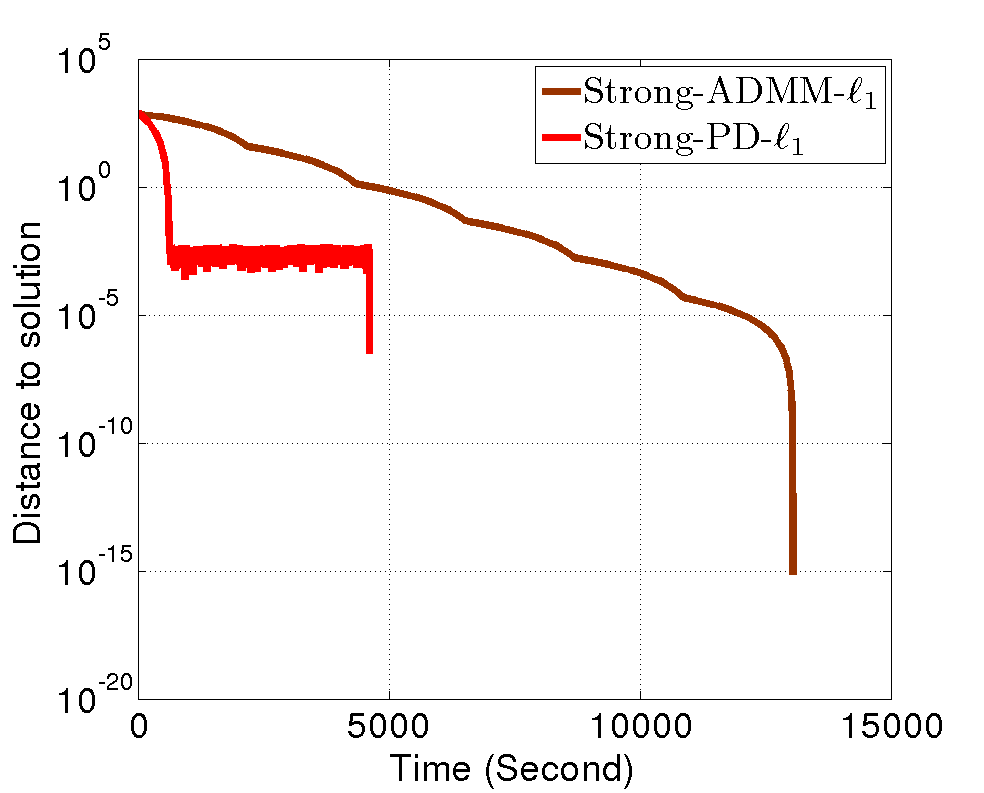}	
	\includegraphics[height=3cm]{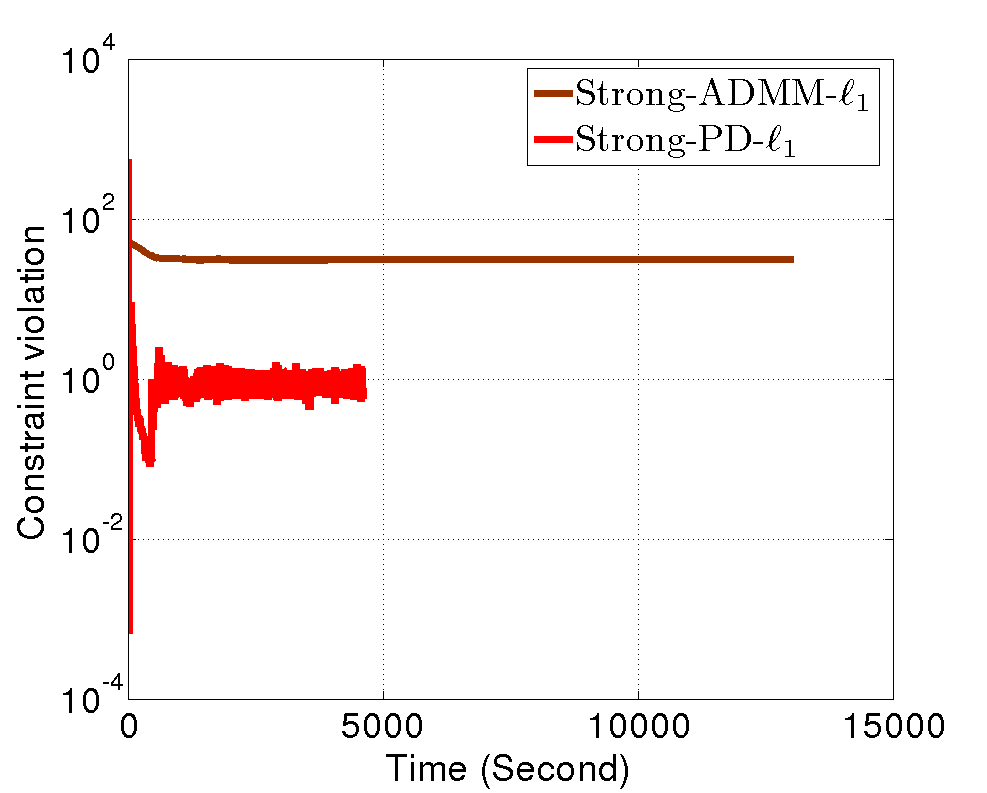}	
	\caption{Comparison between the proposed ``Strong-PD-$\ell_1$''  and ``Strong-ADMM-$\ell_1$'' on Dataset100-030 for $\lambda=10$. (top-left) Objective function \eqref{eq:regform}  w.r.t. iterations, (bottom-left) Objective function \eqref{eq:regform}  w.r.t. time, (top-middle) Distance $\Vert\mathbf{w}^{[k]}-\mathbf{w}^{[\infty]} \Vert$ w.r.t iterations, (bottom-middle) Distance $\Vert\mathbf{w}^{[k]}-\mathbf{w}^{[\infty]} \Vert$ w.r.t. time, (top-right) Distance to the strong hierarchy constraint $S$ w.r.t. iterations, (bottom-right) Distance to the strong hierarchy constraint $S$ w.r.t. time. \label{fig:strongl1compar100-01}}
\end{figure}

\subsubsection{Choice of $\lambda$} \label{sec:similperf}

 {In order to select the optimal $\lambda$, we perform 5 simulations and compute the average performance. Hold-out validation method is adopted for the selection of the trade-off parameter $\lambda$, where the models with different $\lambda$ are trained on the training set, and then the best $\lambda$ is selected based on the performance on the validation set, finally we give the performance of the model on the test set with the selected $\lambda$. We also compare the proposed algorithms to other two common methods:}

 {
\begin{itemize}
 \item SVR-$\ell_2$: it attempts to minimize the criteria with an empirical loss and a $\ell_2$-norm on the weights, aiming to learn small weights, which is expressed as:
 \begin{equation}
   \underset{\mathbf{w}\in \mathbb{R}^{N+N\times N}}{\textrm{minimize}} \lambda\sum_{\ell=1}^L \Big( y_\ell - [ \phi({\bf x}_\ell), \  \phi({\bf x}_\ell)^\top \phi({\bf x}_\ell) ]^\top \mathbf{w} \Big)^2 + \frac{1}{2}\Vert \mathbf{w} \Vert^2 \label{equa:svrl2}
 \end{equation}
 \noindent We apply LIBSVM library~\cite{Chang2011} to solve this problem in the dual form with linear kernels on the training samples. In our experiments, $\lambda$ is validated from $10$ to $1000$ with a space $10$, and the results with the best selected $\lambda$ are shown in Tab.~\ref{tab:simufinalres}. 
 \item SVR-$\ell_1$: it integrates the prior knowledge of sparsity to the criteria and aims to learn a set of sparse weights, which can be written as:
 \begin{equation}
   \underset{\mathbf{w}\in \mathbb{R}^{N+N\times N}}{\textrm{minimize}} \sum_{\ell=1}^L \frac{1}{2} \Big( y_\ell - [ \phi({\bf x}_\ell), \  \phi({\bf x}_\ell)^\top \phi({\bf x}_\ell) ]^\top \mathbf{w} \Big)^2 + \lambda \Vert \mathbf{w} \Vert_1 \label{equa:svrl1}  
 \end{equation}
 \noindent We apply forward-backward proximal algorithm~\cite{Combettes2009} to solve it. In our experiments, $\lambda$ is validated in $[10^{-6}, 10^{-5}, \ldots, 1, 2, 4, \ldots, 2^5]$, and the results with the best selected $\lambda$ are shown in Tab.~\ref{tab:simufinalres}.
\end{itemize}
}

 {The average performance in mean square errors (MSE) for different interaction models with different $\lambda$ on validation set are shown in Fig.~\ref{fig:lambdacomparison}, and their performance with best selected $\lambda$ on the three sets are given in Tab.~\ref{tab:simufinalres}. From the performance point of view in the Fig.~\ref{fig:lambdacomparison} and Tab.~\ref{tab:simufinalres}, it can be seen that: i) The performance SVR-$\ell_2$ are the worst, which is not surprise because the penalization term aims to learn small values, rather than sparse ones. SVR-$\ell_1$ works better and sparse weights are learned, which meets the sparse property of interactions, but it is still inferior to strong hierarchy with $r=1$; ii) ``Weak-PD-$\ell_\infty$'' and ``Weak-FISTA-$\ell_\infty$'' almost have the same performance for different $\lambda$ on both simulation datasets; iii) For Dataset30-345, ``Strong-ADMM-$\ell_1$'' have similar performance with ``Strong-PD-$\ell_1$'', but due to sophisticated hyper-parameter selection in ADMM, fluctuations over different $\lambda$ are observed, especially for large $\lambda$ values; for Dataset100-030, there are some gaps between different $\lambda$, and the performance of ``Strong-ADMM-$\ell_1$'' deteriorates seriously when $\lambda$ becomes large. We conjecture that ADMM would become sensitive to the hyper-parameter selection when the number of features increases. }

\begin{figure}[H]
\centering
	\includegraphics[height=4cm, width=5cm]{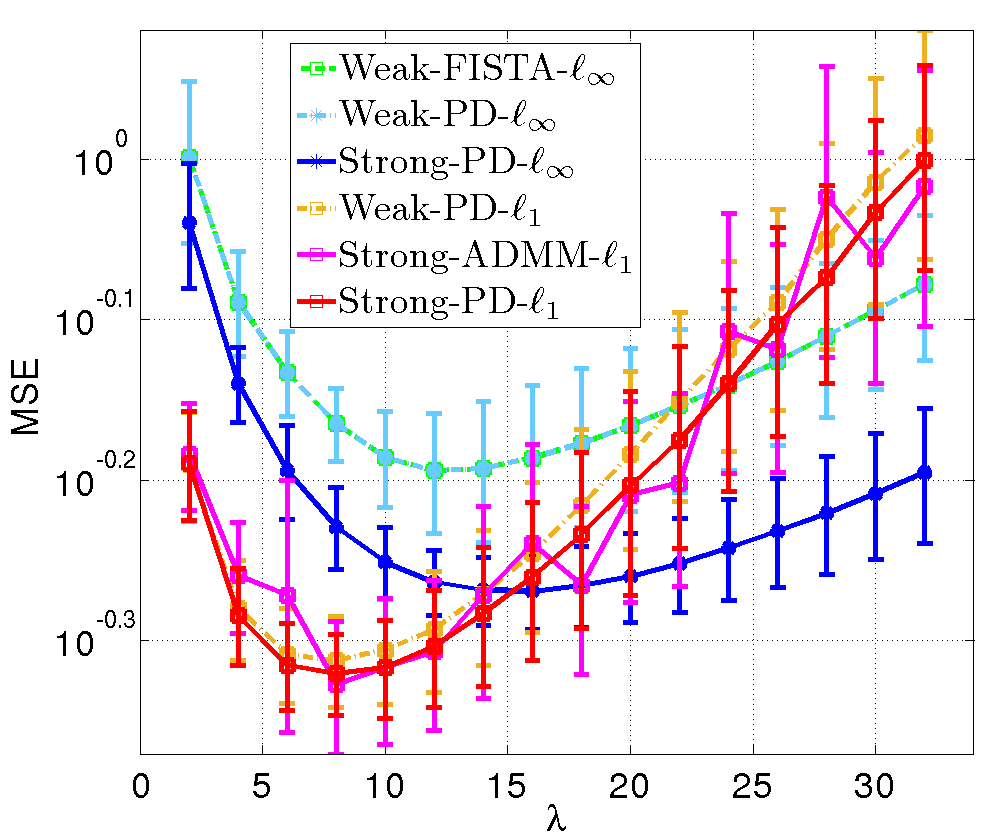} 
	\includegraphics[height=4cm, width=5cm]{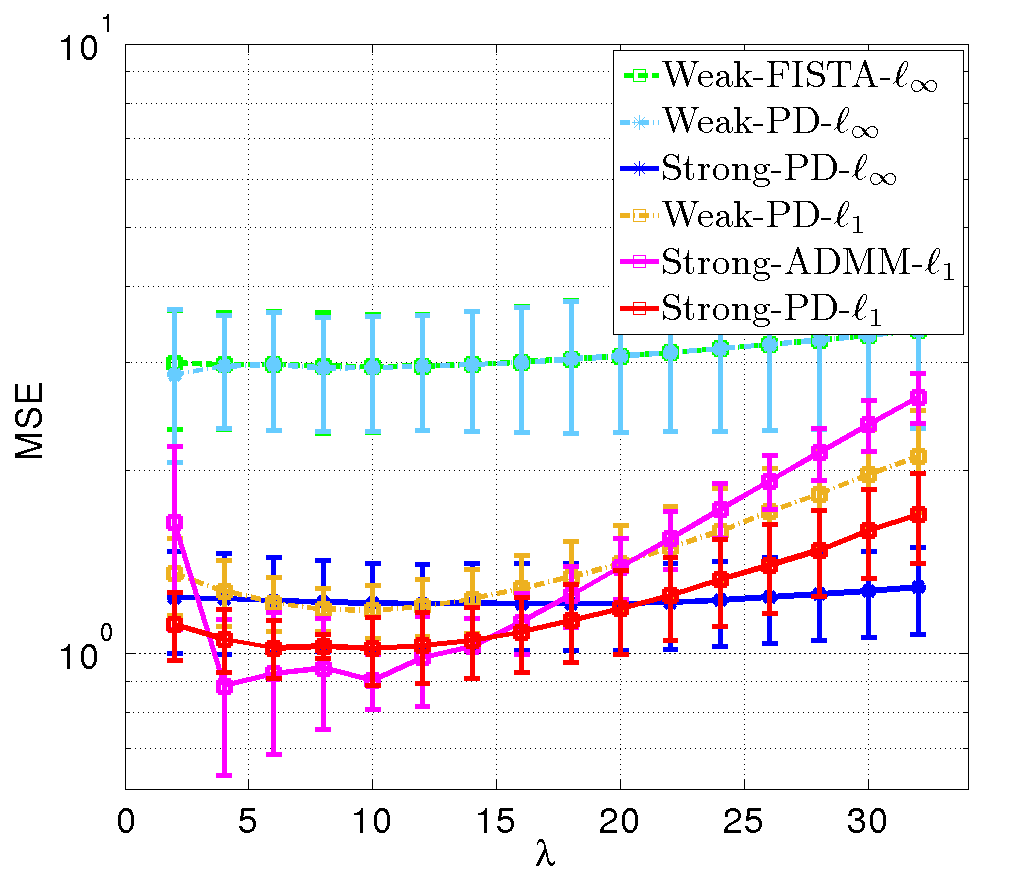} 
	\caption{The performance in MSE on the validation set obtained with the six algorithms (``Weak-FISTA-$\ell_\infty$'', ``Weak-PD-$\ell_\infty$'', ``Strong-PD-$\ell_\infty$, ``Weak-PD-$\ell_1$'',``Strong-ADMM-$\ell_1$ and ``Strong-PD-$\ell_1$) allowing to compare properly the different configurations of the regularization term \eqref{equa:gencons}. (left) Dataset30-345 (right) Dataset100-030).\label{fig:lambdacomparison}}
\end{figure}

\begin{table}[H]
\centering
\resizebox{0.8\textwidth}{!}{
 \begin{tabular}{c|c|c|ccc}
  \hline
 Data &  Approach & $\lambda$ & TR & VAL & TE  \\
\hline
\multirow{8}{*}{Dataset30-347} &  {SVR-$\ell_2$} & 120 & 81.698$\pm$43.091 & 1046.957$\pm$135.988 & 1064.482$\pm$264.993 \\
  &  {SVR-$\ell_1$} & 0.01 & 0.117$\pm$0.013 & 0.520$\pm$0.038 & 0.500$\pm$0.036 \\
  & Weak-FISTA-$\ell_\infty$ & 14 & 0.166$\pm$0.017 & 0.642$\pm$0.065 & 0.602$\pm$0.057 \\
  & Weak-PD-$\ell_\infty$ & 14 & 0.166$\pm$0.017 & 0.642$\pm$0.065 & 0.602$\pm$0.057  \\
  & Strong-PD-$\ell_\infty$ & 16 & 0.190$\pm$0.017 & 0.538$\pm$0.029 & 0.527$\pm$0.048  \\
  & weak-PD-$\ell_1$& 8 & 0.198$\pm$0.017 & 0.487$\pm$0.032 & 0.481$\pm$0.038 \\
  & Strong-ADMM-$\ell_1$ & 8 & 0.184$\pm$0.023 & 0.471$\pm$0.045 & 0.471$\pm$0.061  \\ 
  & Strong-PD-$\ell_1$ & 8 & 0.196$\pm$0.017 & 0.478$\pm$0.028 & 0.473$\pm$0.041  \\
\hline
 \multirow{8}{*}{Dataset100-030} &  {SVR-$\ell_2$} & 90 & 96.217$\pm$61.627 & 1843.295$\pm$62.330 & 1870.404$\pm$174.944  \\
  &  {SVR-$\ell_1$} & 0.1 & 1.496$\pm$0.191 & 2.918$\pm$1.596 & 2.831$\pm$1.255  \\
  & Weak-FISTA-$\ell_\infty$ & 10 & 0.032$\pm$0.001  & 2.954$\pm$0.649 & 3.049$\pm$0.749  \\ 
  & Weak-PD-$\ell_\infty$ & 2 & 0.001$\pm$0.0002 & 2.867$\pm$0.808 & 2.999$\pm$0.943  \\ 
  & Strong-PD-$\ell_\infty$ & 18 & 0.101$\pm$0.004 & 1.034$\pm$0.245 & 1.025$\pm$0.285 \\
  & weak-PD-$\ell_1$ & 10 & 0.109$\pm$0.006 & 1.050$\pm$0.217 & 1.062$\pm$0.251  \\ 
  & Strong-ADMM-$\ell_1$ & 4 & 0.025$\pm$0.001 & 0.884$\pm$0.252 & 0.814$\pm$0.212 \\                       
  & Strong-PD-$\ell_1$ & 10 & 0.113$\pm$0.006  &0.918$\pm$0.179  &0.929$\pm$0.199  \\               
 \hline
 \end{tabular}
 }
 \caption{The comparison results (MSE) on the train, validation and test set of different algorithms under best selected $\lambda$ when both datasets are used.\label{tab:simufinalres}}
\end{table}

\subsubsection{Discussion regarding the choice of $r$} From the above algorithmic comparisons, we have observed that the proposed method, either for $\ell_1$ or $\ell_\infty ,$ delivers an accurate solution (as other algorithmic strategies) but faster and with the possibility to include the strong-hierarchy constraint without inner iterations. In the following analysis, we will thus focus on comparisons between weak/strong $\ell_1$ or $\ell_\infty$ using the proposed algorithm (``Weak-PD-$\ell_\infty$'' , ``Weak-PD-$\ell_1$'' , ``Strong-PD-$\ell_\infty$'' , ``Strong-PD-$\ell_1$''). The average performance with the associated variances are presented in Fig.~\ref{fig:lambdacomparison} for different values of the regularization parameter $\lambda$. Moreover, the performance obtained by hold-out validation are presented in Tab.~\ref{tab:simufinalres}. It can be observed that:
\begin{enumerate}
\item[i)] The good behavior of the strong hierarchy constraint clearly appears for both datasets. Indeed, we recall that the data have  been created with strong hierarchical structure and we can clearly observe that for both datasets the performances with the strong hierarchical constraint are always better either for $r=1$ or $r=+\infty$;
\item[ii)] In our set of experiments, the regularization with $r=1$ always leads to better performance than with $r=+\infty$;
\item[iii)] For ``Weak-PD-$\ell_\infty$'', MSE values associated with the Dataset100-030 are larger, probably due to overfitting. This assumption can be validated by the results obtained on the training dataset.
\end{enumerate}

\subsection{Application to HIV Data} \label{sec:hivsection}
 {Feature interaction learning is very meaningful to discover the intrinsic correlations between the variables, especially in the real-life applications, for example, genetic diseases and drug analysis. Since the effects usually depends on several genes or factors, the interaction of these associating factors is worth studying.} In this section, we apply the proposed algorithms to drug analysis: HIV dataset on the susceptibility of the HIV-1 virus to six nucleoside reverse transcriptase inhibitors (NRTIs). This dataset\footnote{It can be downloaded from https://hivdb.stanford.edu/pages/published\_analysis/genophenoPNAS2006/} was collected by \cite{Rhee14112006} and it was used to model HIV-1 susceptibility to the drugs because HIV-1 virus can become resistant through genome mutations for different subjects. In the dataset, there are 639 subjects and 240 gene locations. For each observation, the mutation status at each gene location is recorded and also give six drugs (log) susceptibility responses. In this work, we focus on the prediction model for 3TC drug. 

 {The dataset is randomly split into two half sets for training and test, respectively as adopted by \cite{Bien2013,Haris2014}. The trade-off parameter $\lambda$ is selected similarly by hold-out validation method. The performance are measured by the average MSE with the associating variance over 5 random splits on the test set for ``Weak-PD-$\ell_\infty$'', ``Strong-PD-$\ell_\infty$'', ``Weak-PD-$\ell_1$'', ``Strong-PD-$\ell_1$'', ``Weak-FISTA-$\ell_\infty$'' and ``Strong-ADMM-$\ell_1$''.}

\subsubsection{Reduced dataset}  Following \cite{Haris2014}, we first work on a reduced dataset over the bins of ten adjacent loci, rather than all 240 genes locations, resulting from the fact that nearby genomes have similar effects to the drug susceptibility, and also leading to less sparse data. So for the first set of experiments we have $N=24$ features. The value for each bin is set to 1 if one of  {genes} in that bin undergoes mutation. 

\begin{figure}[tp]
\centering
  \includegraphics[height=4cm]{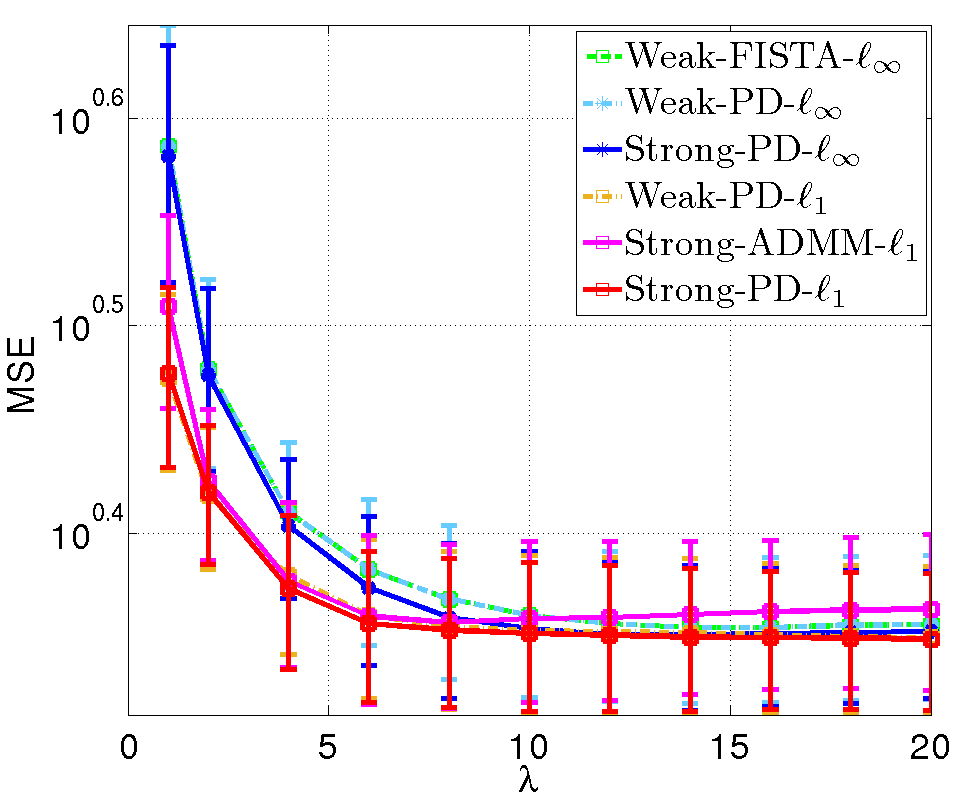} 
  \includegraphics[height=4cm]{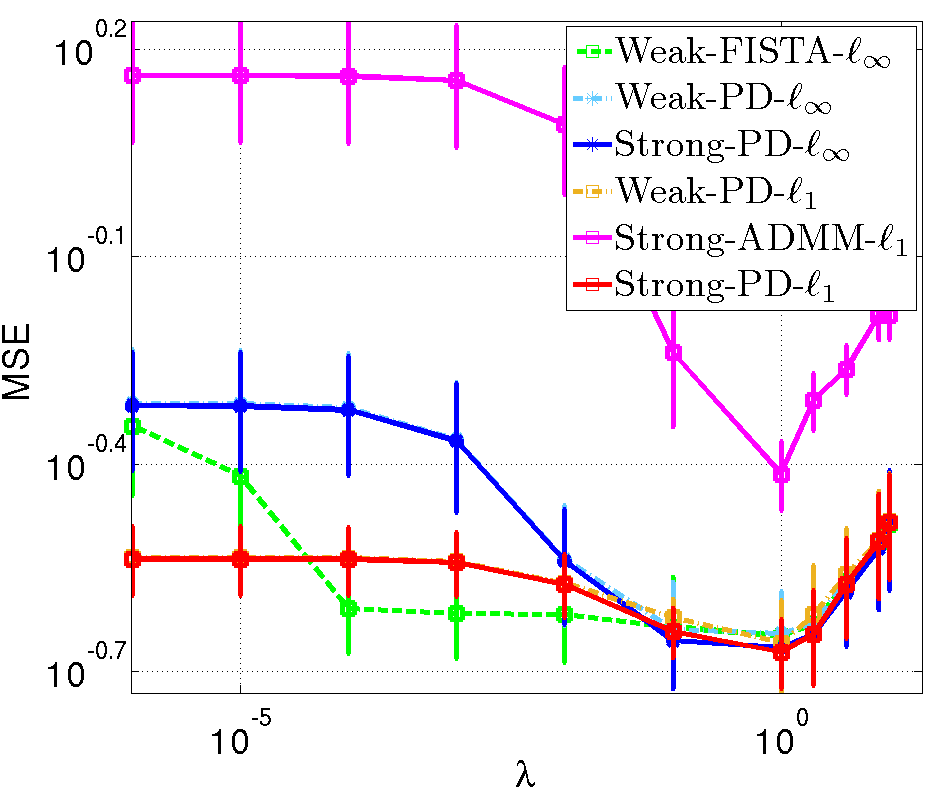} 
  \caption{Performance comparisons between ``Weak-FISTA-$\ell_\infty$'', ``Weak-PD-$\ell_\infty$'', ``Strong-PD-$\ell_\infty$, ``Weak-PD-$\ell_1$'',``Strong-ADMM-$\ell_1$ and ``Strong-PD-$\ell_1$ for different $\lambda$ values on the test set of HIV dataset when $N=24$ (left) and $N=240$ (right). \label{fig:hivresults}}
\end{figure}

\begin{figure*}
\centering
\begin{tabular}{ccc}
	\includegraphics[height=3cm, width=3.5cm]{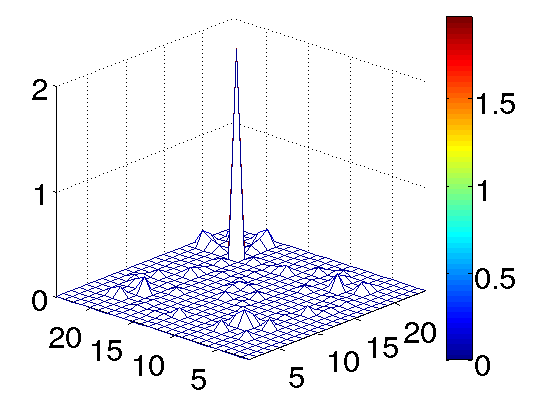} &
	\includegraphics[height=3cm, width=3.5cm]{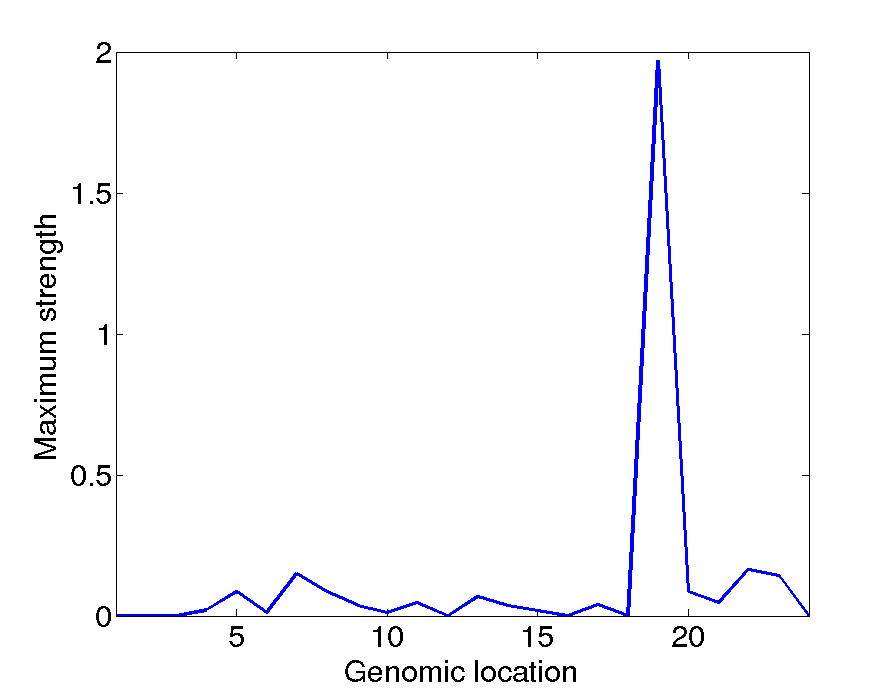} &
	\includegraphics[height=3cm, width=3.5cm]{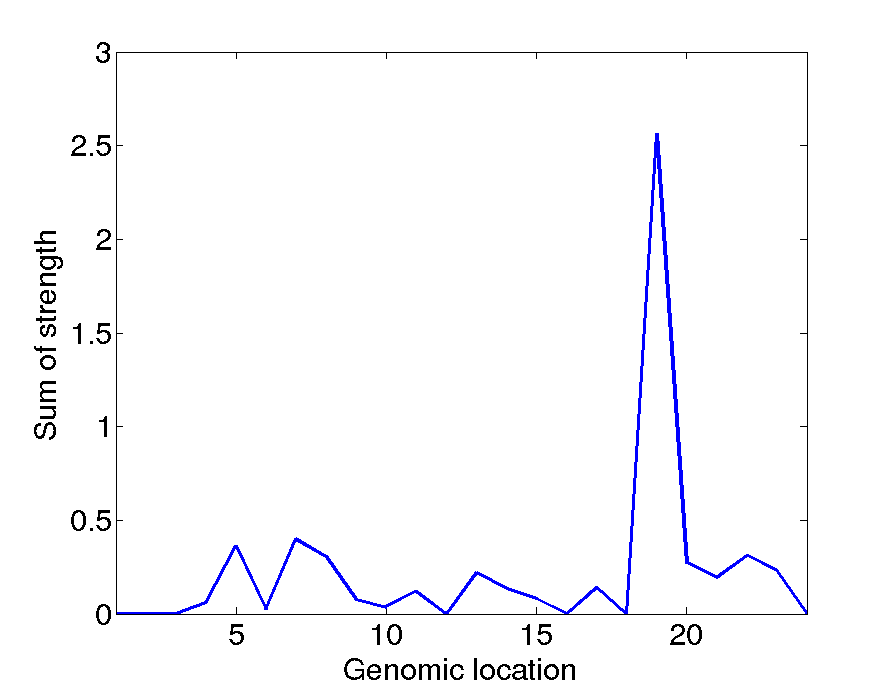} 
	\\	
	\includegraphics[height=3cm, width=3.5cm]{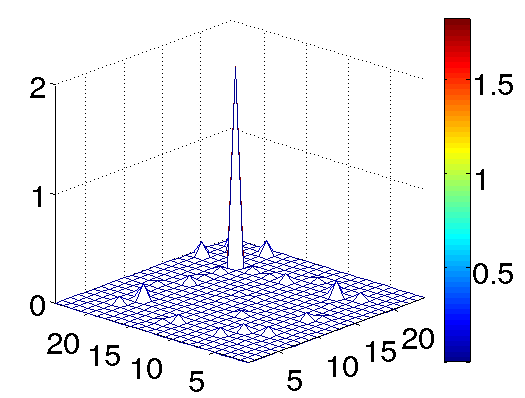}&
	\includegraphics[height=3cm, width=3.5cm]{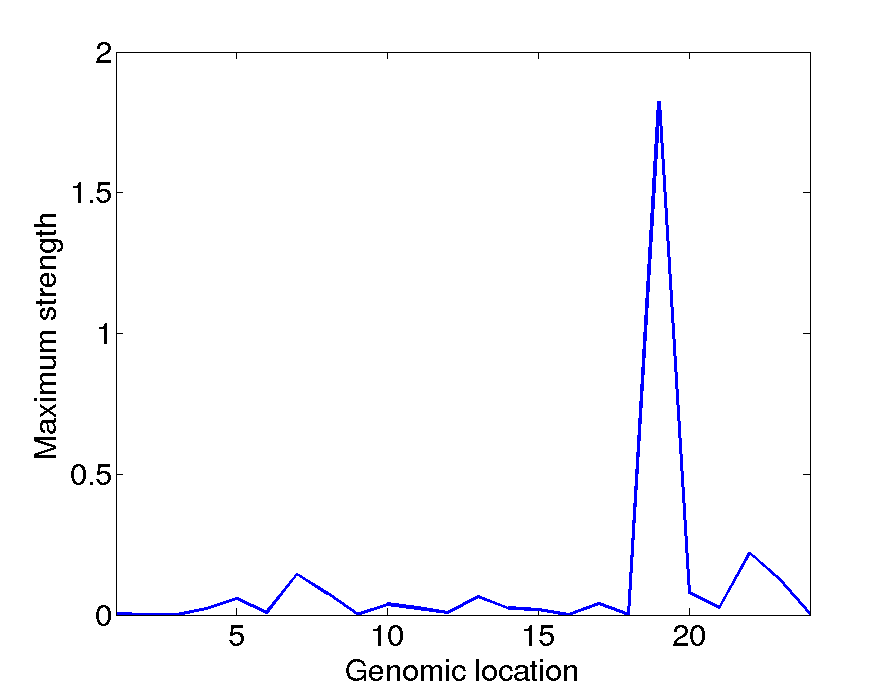}	&
	\includegraphics[height=3cm, width=3.5cm]{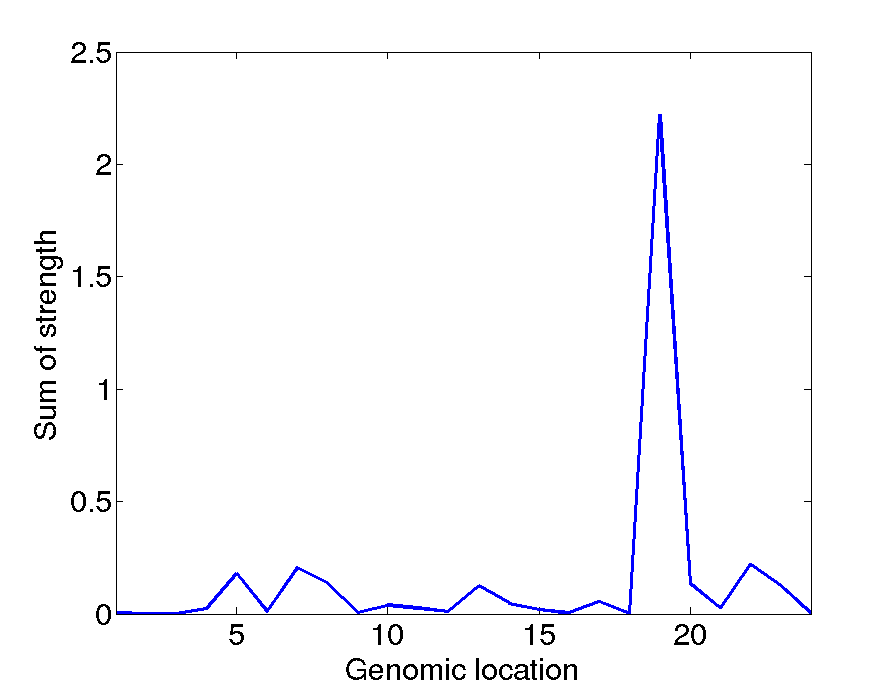}
\end{tabular}		
	\caption{The figures in the first row shows the interaction (left column), $\ell_\infty$-norm over each row of interactions (middle column), $\ell_1$-norm over each row of interaction (right column) when $N=24$, $r=\infty$ and $\lambda=12$. The figures in the second row shows the corresponding results when when $N=24$, $r=1$ and $\lambda=10$. \label{fig:hivresult-analysis}}
\end{figure*}

\begin{figure*}
\centering
	\includegraphics[height=3cm, width=3.5cm]{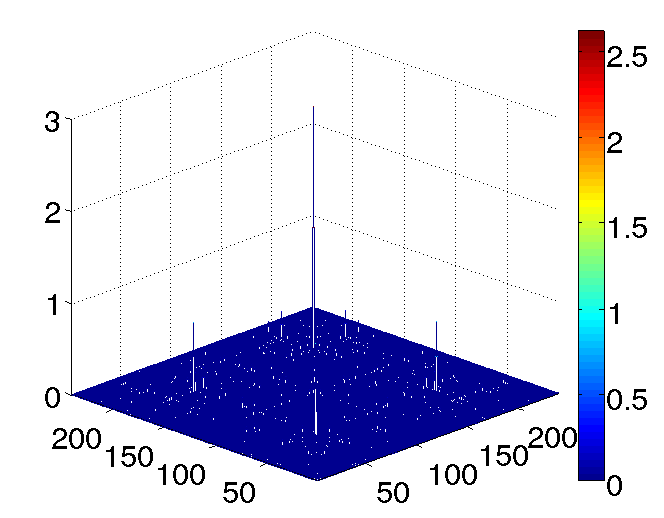} 
	\includegraphics[height=3cm, width=3.5cm]{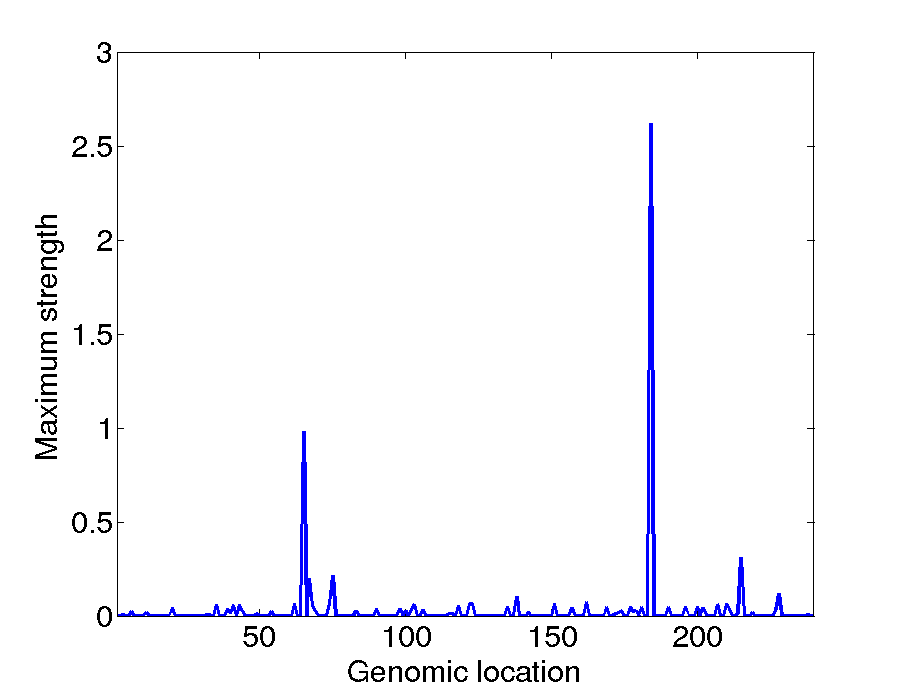} 
	\includegraphics[height=3cm, width=3.5cm]{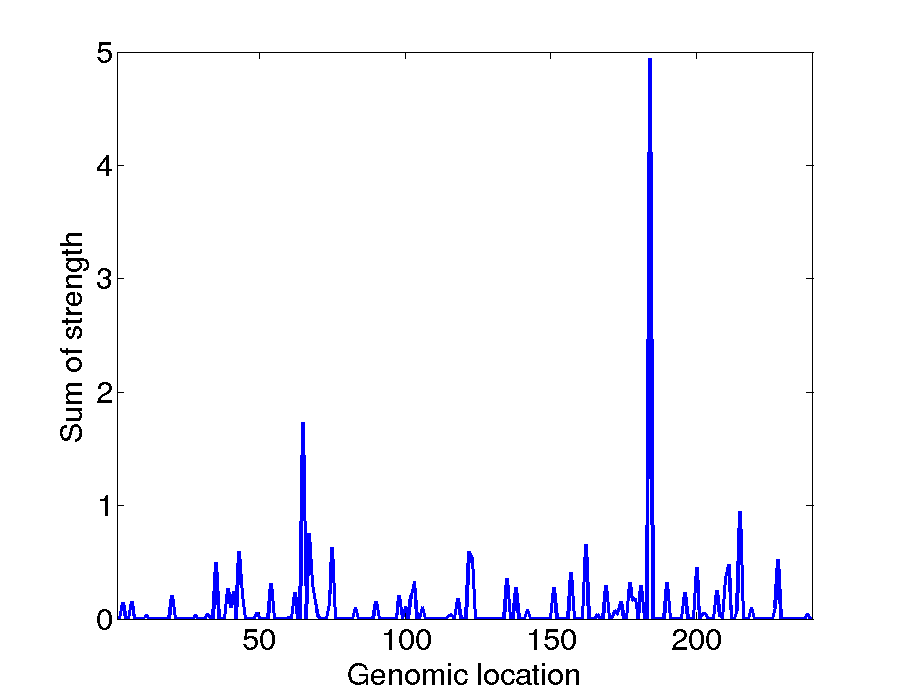} 
	\\	
	\includegraphics[height=3cm, width=3.5cm]{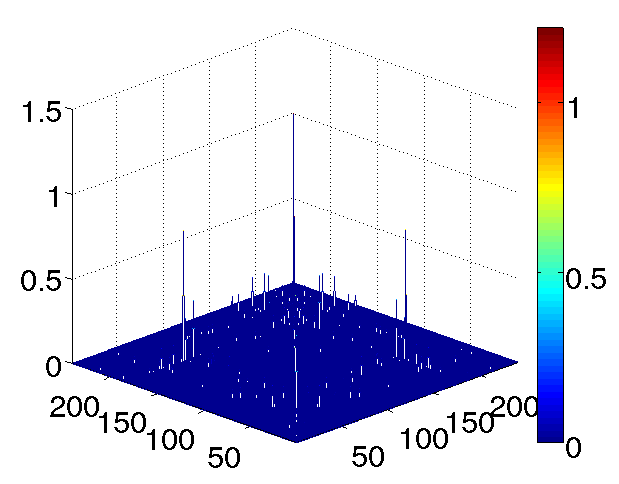}
	\includegraphics[height=3cm, width=3.5cm]{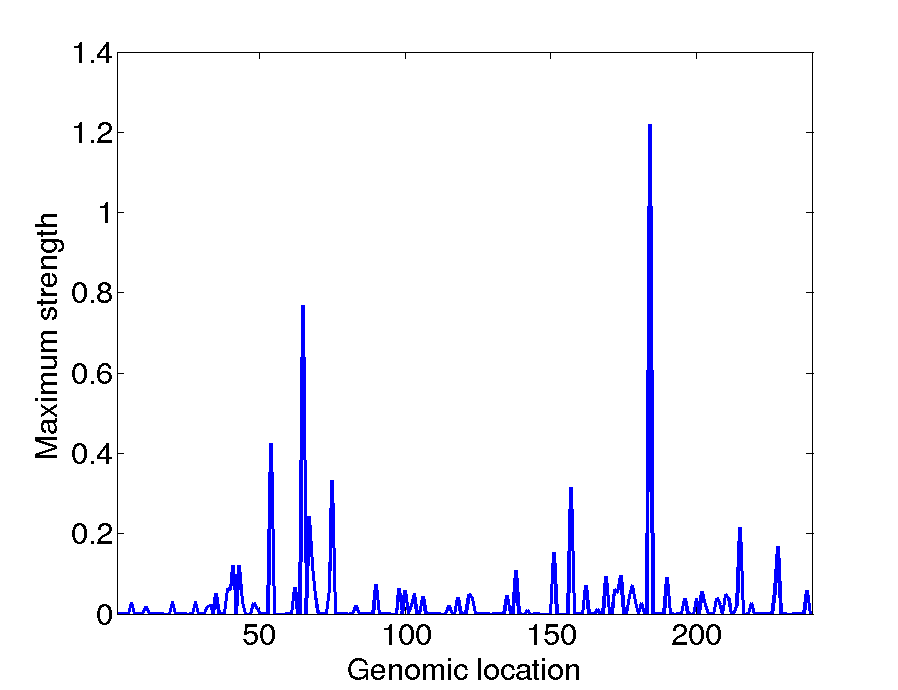}	
	\includegraphics[height=3cm, width=3.5cm]{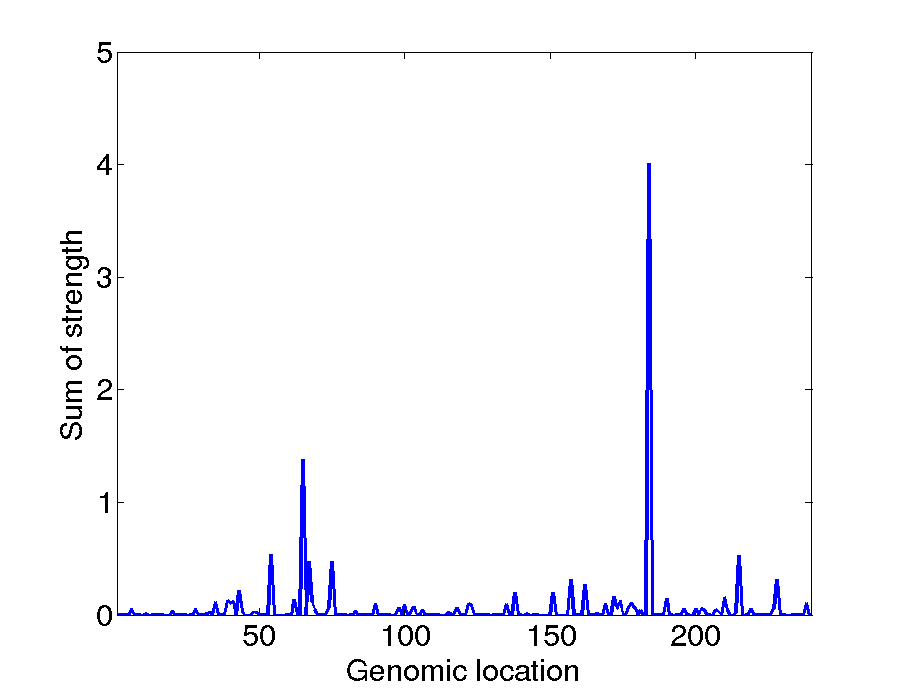}		
	\caption{The figures in the first row shows the interaction (left column), $\ell_\infty$-norm over each row of interactions (middle column), $\ell_1$-norm over each row of interactions (right column) when $N=240$, $r=\infty$ and $\lambda=1$. The figures in the second row shows the corresponding results when when $N=240$, $r=1$ and $\lambda=1$. \label{fig:hivresult-analysis-large}}
\end{figure*}

 {Fig.~\ref{fig:hivresults} (left) shows the performance of six algorithms on the reduced dataset.} It can be observed that:  {i) ``Weak-FISTA-$\ell_\infty$ has the same performance with ``Weak-PD-$\ell_\infty$ over different $\lambda$; ii) a slight better performance are obtained for ``Strong-PD-$\ell_1$'' compared to ``Strong-ADMM-$\ell_1$''; } iii) ``Strong-PD-$\ell_\infty$'' and ``Strong-PD-$\ell_1$'' are slightly better than their weaker counterparts, and the best $\lambda$ for ``Strong-PD-$\ell_\infty$'' and ``Strong-PD-$\ell_1$'' are 12 and 10 respectively. The interactions  {over one split} are visualized in the Fig.~\ref{fig:hivresult-analysis} (left). In order to better visualize the effect of the penalizations, we also plot the $\ell_\infty$-norm and $\ell_1$-norm over each row of $\Theta$ for strong hierarchy, as shown in the middle and right of Fig.~\ref{fig:hivresult-analysis}. Both have the strong effect for 19$^{th}$ feature, which is consistent with the results observed in \cite{Haris2014}. We also observe an interesting fact from the visualization of interactions, it can be seen that ``Strong-PD-$\ell_1$'' is able to learn a sparser interaction matrix than ``Strong-PD-$\ell_\infty$'', however, the maximum strength (i.e. $\ell_\infty$-norm) and the sum of the strength ($\ell_1$-norm) in Fig.~\ref{fig:hivresult-analysis} have similar distributions for both cases. 
The possible reason is that the extra interactions learned by ``Strong-PD-$\ell_\infty$'' are subtle, it can be also confirmed that both performances are very similar in the Fig.~\ref{fig:hivresults} (left).

\subsubsection{Full dataset} We further work on the full data without binning operation leading to $N=240$ features and each feature gives the indicator of mutation. The data splitting is the same as the one previously described. The average performance of MSE on the test set for six algorithms are shown in Fig.~\ref{fig:hivresults} (right). It is clear to demonstrate that:  {i) on large dataset, there are some performance gaps between FISTA, ADMM and their primal-dual counterparts for weak and strong hierarchy over different $\lambda$, especially ADMM deteriorates seriously and the performance becomes worse; ii) strong hierarchy still produces better performance than the weak one.} The learned interactions and their feature effects from ``Strong-PD-$\ell_\infty$'' and ``Strong-PD-$\ell_1$'' with best $\lambda=1$  {over one split} are shown in Fig.~\ref{fig:hivresult-analysis-large}. It is found that both algorithms can detect that 184$^{th}$ genes location has the most strong effect. From the maximum strength and sum of strength for each gene in the Fig.~\ref{fig:hivresult-analysis-large}, we can see the different properties for both algorithms. 
For ``Strong-PD-$\ell_\infty$'', maximum strength for each gene has a more sparse distribution than the sum of strength, which is consistent with its objective that the maximum of absolute values of $\Theta^{(i,:)}$ is not larger that $v^{(i)}$, whereas a more sparse distribution over the sum of strength is 
obtained for ``Strong-PD-$\ell_1$'', whose objective is to ensure that the sum of absolute values of $\Theta^{(i,:)}$ is not larger that $v^{(i)}$. From Fig.~\ref{fig:hivresults}(right), it also can be observed that ``Strong-PD-$\ell_1$'' is slightly better than ``Strong-PD-$\ell_\infty$'' because it imposes a stronger penalization for $\mathbf{\Theta}$ than ``Strong-PD-$\ell_\infty$''.

\subsection{Multiclass learning} \label{sec:experimentsclass}

In \cite{Jiumlsp2018}, we performed preliminary validation of the proposed method compared to state-of-the-art in the context of face classification. In this work, we tackle  Parkinson disease classification from the speech records of the persons. Thus, $\mathbf{x}_\ell$ models the speech data (1D) and $\phi(\mathbf{x}_\ell)$ the time-frequency based features. For all the experiments, we compare the performance in term of accuracy between:
\begin{itemize}
\item sparse multiclass SVM without interaction, i.e., $f_k(\mathbf{x}_\ell) =  \mathbf{v}_k^{\top} \phi(\mathbf{x}_\ell)$ and $\Omega(\mathbf{v}_k) =  \sum_k\Vert \mathbf{v}_k\Vert_1$,
\item sparse multiclass SVM with full interactions: $f_k(\mathbf{x}_\ell) = \phi(\mathbf{x}_\ell)^{\top} \mathbf{\Theta}_k \phi(\mathbf{x}_\ell) + \mathbf{v}_k^{\top} \phi(\mathbf{x}_\ell)$ and $\Omega(\mathbf{v}_k,  \mathbf{\Theta}_k )= \Vert \mathbf{v}_k \Vert_1 + \sum_i\Vert \mathbf{\Theta}_k^{(i,\cdot)} \Vert_1$,
\item proposed approach that is sparse multiclass SVM with strong hierarchy. 
\end{itemize}
The full interaction case differs from the proposed one in the sense that there is no coupling between $\mathbf{v}_k$ and $\Theta_k$ in Eq.~(\ref{eq:equivform}).\\

 {

\noindent \textbf{Dataset} -- We first show the Parkinson disease classification by sparse multiclass SVM with strong hierarchy. The experiments are conducted on the Parkinson Speech dataset~\cite{Sakar2013}, which is used to study the underlying relationship between the Parkinson disease and the types of voices of the patients. This dataset consists of training set and test set. For the training set, the data are collected from 20 Parkinson patients (6 female, 14 male) and 20 healthy individuals (10 female, 10 male), for each subject, 26 sound samples including sustained vowels, words, and short sentences are recorded, containing in total 1040 training samples. 26 type of time-frequency based features (such as jitter, shimmer, median pitch and number of pulses etc.) are extracted for each sound samples. For the test set, the data are only collected from another 28 Parkinson patients. They are asked to say particular sustained vowels (`a' and `o') three times and the same time-frequency based features are extracted, therefore it contains 168 samples. In our experiments, we only use the training set, because the test set does not contain negative samples. We randomly divide the training set to three subsets of equal size for model training, validation and test. The features are preprocessed by Gaussian normalization and $10^6$ number of iterations are adopted to guarantee convergence in the learning stage.
}

 {
\noindent \textbf{Performance} -- We first show the results of multiclass SVM without interaction in Tab.~\ref{tab:parkinson-l1-result}. The accuracy on the training, validation and test set with different $\lambda$ and also the non-zero rate of $\mathbf{v}$ are given. From the results, it is seen that a sparse solution is obtained with $\lambda=10^{-2}$, as well as it delivers the best test accuracy (63.98\%) on the test set.  
}

\begin{lrbox}{\results}
 \begin{tabular}{|c|c|c|c|c|c|c|} 
  \hline
    & \multicolumn{6}{c|}{$\lg (\lambda)$} \\
   \cline{2-7}
    & -1 & -2 & -3 & -4 & -5 & -6  \\
   \hline
    Train acc. & 62.72 & \textbf{65.90} & 68.50 & 69.36 & 69.08 & 69.08\\
    Val. acc. & 58.79 & \textbf{60.23} & 59.94 & 58.79 & 58.50 & 58.50\\   
    Test acc. & 62.54 & \textbf{63.98} & 62.82 & 61.96 & 61.96 & 61.96\\
    NZ rate in $\mathbf{v}$ & 23.08 & \textbf{65.38} & 92.31 & 92.31 & 92.31 & 100 \\  
   \hline
  \end{tabular}
\end{lrbox}

\begin{table}[h]
  \centering
  \scalebox{0.8}{\usebox{\results}} 
  \caption{The classification accuracy of multiclass SVM with $\ell_1$-norm (in \%) on Parkinson Speech dataset. \label{tab:parkinson-l1-result}}
 \end{table}
 \begin{figure}[h]
  \centering
  \includegraphics[scale=0.28]{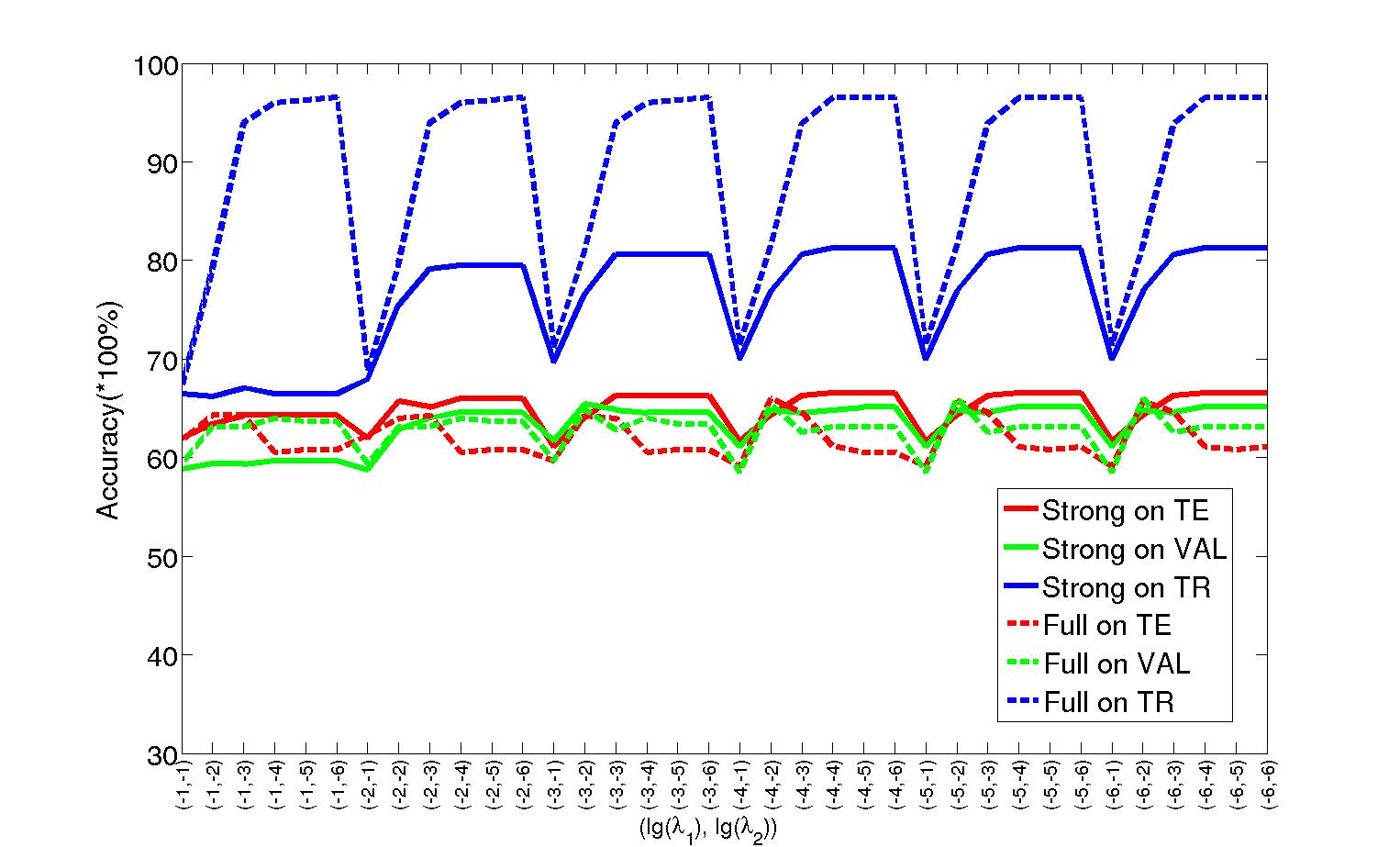}
  \caption{This figure shows classification accuracy comparison on the training, validation and test for multiclass SVMs with full interactions and strong hierarchy under different combinations of $\lambda_1$ and $\lambda_2$ on Parkinson Speech dataset. \label{fig:parkinson-intercation-results}} 
\end{figure}

 {
Next we show the results of multiclass SVMs with full interaction and strong hierarchy, taking into account the interactions of features. Fig.~\ref{fig:parkinson-intercation-results} shows the performance comparison curves under different combinations of $\lambda_1$ and $\lambda_2$ for multiclass SVMs with full interactions and strong hierarchy. It is observed that: 1) the accuracy with full interaction on the training set are much higher than multiclass SVMs with $\ell_1$-norm and strong hierarchy, but the accuracy on the validation and test set are not reached to as much as on the training set, the possible reason is that by combining interactions of features without further constraints, overfitting problem occurs when the learned weights (i.e.~$\mathbf{v}$ and $\mathbf{\Theta}$) are increased; 2) when strong hierarchy is regularized to $\mathbf{v}$ and $\mathbf{\Theta}$, the overfitting is alleviated, and the training accuracy are decreased, while the accuracy on the validation and test set are improved compared to the one with full interactions; 3) From Tab.~\ref{tab:parkinson-l1-result} and Fig.~\ref{fig:parkinson-intercation-results}, the best accuracy on the test set for multiclass SVMs with $\ell_1$-norm is 63.98\% ($\lambda=10^{-2}$), while the best accuracy for full interaction and strong hierarchy are 65.99\% ($\lambda_1=10^{-4}$, $\lambda_2=10^{-2}$) and 66.57\% ($\lambda_1=10^{-4}$, $\lambda_2=10^{-4}$) respectively, which validates that strong hierarchy can help to improve the classification performance.
} 

\vspace{-0.3cm}
\section{Conclusion} \label{sec:conclusion}
In this work, we propose a primal-dual proximal algorithm with epigraphical projection for regression and multiclass SVMs with strong hierarchy and apply it to several regression and classification task.  Four algorithms are derived (``Weak-PD-$\ell_\infty$'' , ``Weak-PD-$\ell_1$'' , ``Strong-PD-$\ell_\infty$'' , ``Strong-PD-$\ell_1$'') allowing to deal with and to  properly compare different configurations of the regularization term \eqref{equa:gencons}.  Compared to other state-of-the-art algorithms, for instance, FISTA or ADMM, the proposed algorithm is computationally more efficient (finding a solution belonging to $S$), and convergence in terms of iterates is faster. Compared to standard sparse learning strategies, we can integrate the main effects and interaction effects, and then obtain the underlying graphs between the features, as a result, the regularization with strong hierarchy make them more discriminative.

\section*{Appendix A: Proof of Proposition~\ref{prop:epi}} \label{appen:proposition1}

The proof follows arguments derived in \cite{Bien2013} for the specific case $r=1$.  Let $\mathbf{v}=\mathbf{v}^{+}-\mathbf{v}^{-}$, \eqref{eq:equivform} can be equivalently written as
\begin{align}
 \underset{\mathbf{v},\mathbf{v}^+, \mathbf{v}^-, \mathbf{\Theta}  }{\textrm{minimize}} & \; \frac{1}{2} \sum_{\ell=1}^L  \big( y_\ell  -   \phi^\top({\bf x}_\ell) \mathbf{v}- \phi^\top({\bf x}_\ell) \mathbf{\Theta} \phi({\bf x}_\ell) \big)^2 + \frac{\lambda}{2} \Vert  \mathbf{\Theta}  \Vert_1 +    \lambda 1^\top(2 \mathbf{v}^+ - \mathbf{v})  \nonumber\\
&\!\! \!\!\!\!\!\! \!\!\!\!\!\! \!\!\!\!\!\! \!\!\!\!\!\! \!\!\!\!\mbox{subj. to} \quad 
\begin{cases}
 \Vert \Theta^{(i,\cdot)}\Vert_{r} \leq  2 v^{+(i)} - v^{(i)}  \quad (\forall i\in \{1,\ldots,N\})\\
\mathbf{v}=\mathbf{v}^{+}-\mathbf{v}^{-}\\
(\mathbf{v}^+, \mathbf{v}^-,\Theta)\in O\times O \times C.
\end{cases}
\end{align}
The constraints involving $(\mathbf{v},\mathbf{v}^+, \mathbf{v}^-)$ can be reformulated as
\begin{equation} \nonumber
(\forall i\in \{1,\ldots,N\}) \quad  \begin{cases} v^{+(i)} \geq \frac{1}{2} (\Vert \Theta^{(i,\cdot)}\Vert_{r} + v^{(i)}) ,\\ 
v^{+(i)} \geq v^{(i)}  \\ 
v^{+(i)} \geq 0.\end{cases}
\end{equation}
yielding $v^{+(i)} \geq \max \{ \max\{0, v^{(i)}\}, \ \frac{1}{2} (\Vert \Theta^{(i),\cdot}\Vert_{r} + v^{(i)}) \}$. Using $\vert v^{(i)} \vert = 2\max\{0, v^{(i)}\} - v^{(i)}$, we get the expected result:
\begin{align}
 \underset{\mathbf{v},\mathbf{\Theta}  }{\textrm{minimize}}  & \; \frac{1}{2} \sum_{\ell=1}^L  (y_\ell  -   \phi^\top({\bf x}_\ell) \mathbf{v} - \phi^\top({\bf x}_\ell) \mathbf{\Theta} \phi({\bf x}_\ell))^2 \nonumber\\ & +  \frac{\lambda}{2} \Vert \mathbf{\Theta}  \Vert_1 + \lambda \sum_{i=1}^N \max \{  \vert v^{(i)} \vert,  \Vert \Theta^{(i,\cdot)} \Vert_{r} \} + \iota_C(\mathbf{\Theta} ). \label{equa:hier}
\end{align}

\section*{Appendix B: Proof of Proposition~\ref{propo:inftyprojection}} \label{appen:proposition2}

According to the definition of the epigraphical projection, we solve the following minimization problem:
\begin{multline}
P_{E_\infty}(\omega^+, \omega^-, \mathbf{u}) \\
\;\;\;\;=\underset{\eta^+, \eta^-\geq0}{\arg\min} \;(\eta^+ - \omega^+)^2 + (\eta^- - \omega^-)^2 
 + \!\!\!\min_{\substack{\vert p^{(1)} \vert < \eta^+ + \eta^-\\\cdots\\\vert p^{(N)} \vert < \eta^+ + \eta^- }} \Vert \mathbf{p} - \mathbf{u} \Vert^2\!\!\!   \label{equa:inftyepi1}
\end{multline}

\noindent The inner minimization yields a simple projection for $\mathbf{u}$ to $[-\eta^+ - \eta^-, \eta^+ + \eta^-]$, i.e. $\mathbf{p} = P_{[-\eta^+ - \eta^-,  \eta^+ +\eta^-]}(\mathbf{u})$, thus Eq.~(\ref{equa:inftyepi1}) becomes:
\begin{align}
P_{E_\infty}(\omega^+, \omega^-, \mathbf{u}) &=  \underset{\eta^+, \eta^-\geq0} {\arg\min} \Bigg\{(\eta^+ - \omega^+)^2 + (\eta^- - \omega^-)^2  \nonumber\\
& \quad+  \sum_{i=1}^N (\max \{\vert u^{(i)} \vert - \eta^+ - \eta^-, 0 \})^2\Bigg\}
\end{align}
and thus
\begin{equation}
P_{E_\infty}(\omega^+, \omega^-, \mathbf{u}) =  \textrm{prox}_{\phi+\iota_{O}} (\omega^+, \omega^-)  
\end{equation}
\noindent where $\phi(\omega^+, \omega^-) =  \sum_{i=1}^N (\max \{ \vert u^{(i)} \vert - \eta^+ - \eta^-, 0 \})^2$. We now focus on the computation of $\textrm{prox}_{\phi}(\omega^+, \omega^-)$. \\
\indent First, we sort $(\vert u^{(i)}\vert)_{1\leq i \leq N}$ in  ascending order to be $(\mathbf{\nu}^{(1)},\ldots,\mathbf{\nu}^{(N)})$, such that $\vert \mathbf{\nu}^{(1)} \vert \leq \ldots \leq \vert \mathbf{\nu}^{(N)} \vert$, and also $\mathbf{\nu}^{(0)}=-\infty$; and second $ \bar{n} \in [1, N]$ can be found such that $\mathbf{\nu}^{(\bar{n}-1)} \leq \eta^+ + \eta^-  \leq \mathbf{\nu}^{(\bar{n})}$, then $\phi(\omega^+, \omega^-) = \sum_{i=\bar{n}}^M (\eta^+ + \eta^- - \mathbf{\nu}^{(\bar{n})})^2$,  so the proximity operator of $\phi$ has:
\begin{equation}
\begin{cases}
\mathbf{\nu}^{(\bar{n}-1)} \leq \eta^+ + \eta^- \leq \mathbf{\nu}^{(\bar{n})}, \\
\omega^+ - \eta^+ = (N-\bar{n}+1)(\eta^+ + \eta^-) - \sum_{i=\bar{n}}^N \mathbf{\nu}^{(i)}, \\
\omega^- - \eta^- = \omega^+ - \eta^+,\\
\end{cases}
\end{equation}
and that leads to
\begin{equation}
\begin{cases}
 \eta^- = \frac{\omega^- - (N - \bar{n}+1)(\omega^+ - \omega^-) + \sum_{i=\bar{n}}^N  \mathbf{\nu}^{(i)}}{1+2(N-\bar{n}+1)},\\
 \eta^+ = \eta^- + \omega^+ - \omega^-. 
\end{cases}
\end{equation}

\section*{Appendix C: Proof of Proposition~\ref{propo:epil1preli}} \label{appen:proposition3}

The Lagrangian associated to the function involved in $P_{E_1^+}$ is:
\begin{multline}
 \mathcal{L}(\eta^+, \eta^-, \mathbf{p},\alpha,  \boldsymbol{\xi}) = \frac{1}{2}\Vert {\mathbf{p}} - \mathbf{u} \Vert^2 + \frac{1}{2}(\eta^+ - \omega^+)^2  \\
+ \frac{1}{2}(\eta^- - \omega^-)^2+ \alpha\big(\sum_{i=1}^N {{p}}^{(i)} - \eta^+ - \eta^-\big) -   \boldsymbol{\xi}^\top {\mathbf{p}}.
\end{multline}
The KKT conditions are:
\begin{equation} \label{equa:larl1norm}
 \begin{cases}
(\forall i \in \{1, \ldots, N \})\quad {p}^{(i)} - {u}^{(i)} + \alpha - \mathbf{\xi}_i= 0  \\
\eta^+ - \omega^+ - \alpha = 0 \\
\eta^- - \omega^- - \alpha = 0\\
(\forall i \in \{1, \ldots, N \})\quad \xi^{(i)} {p}^{(i)}  = 0 \\
\sum_{i=1}^N {{p}}^{(i)} - \eta^+ - \eta^- = 0. \\
  \end{cases}
\end{equation}
From the fourth condition in Eq.~\eqref{equa:larl1norm}, we know that if ${p}^{(i)} >0$ then $\xi^{(i)} =0$ and $p^{(i)} = {u}^{(i)} - \alpha$. Thus, if we denote $\boldsymbol{\pi} = (\mathbf{\pi}^{(i)} )_{1\leq i \leq N}$ is an ordered version of $( p^{(i)})_{1\leq n \leq N}$ in a decreasing order, we can write
\begin{equation}
 \sum_{i=1}^N p^{(i)} =   \sum_{i=1}^N \pi^{(i)} =   \sum_{i=1}^{\widetilde{n}} \pi^{(i)} =  \sum_{i=1}^{\widetilde{n}} (\mathbf{\mu}^{(i)} - \alpha)  =\eta^+ + \eta^-.
\end{equation}
From similar arguments as in \cite{Duchi2008}[Lemma 2] , $\widetilde{n}$ is computed as \eqref{eq:ntilde} and thus 
\begin{equation}
  \alpha = \frac{1}{\widetilde{n}} \Big(\sum_{i=1}^{\widetilde{n}} \mathbf{\mu}^{(i)} -(\eta^+ + \eta^-)\Big).
\label{eq:alpha}
\end{equation}

From the second and third equality of the KKT conditions, we have $\eta^+ = \eta^- + \omega^+ - \omega^-$. Finally, combining  \eqref{eq:alpha} with $\eta^- = \omega^- + \alpha$, yields
\begin{equation}
 \eta^- = \frac{1}{\widetilde{n}(1+2/\widetilde{n})} \Big( \sum_{i=1}^{\widetilde{n}} \mathbf{\mu}^{(i)} - \omega^+ + (\widetilde{n}+1) \omega^-\Big).
\end{equation}

\begin{acknowledgements}
The authors would like to thank Marie-France Benassy, Mickael Bonnard, Laurent Grosset, Georges Oppenheim, Maxime Durot, Leire Oro-Urea from TOTAL for helpful discussion, Prof. Mark Asch from Universit\'{e} de Picardie Jules Verne for the comments. 
\end{acknowledgements}

\end{document}